\documentclass[sn-mathphys-num]{sn-jnl}

\usepackage{graphicx}%
\usepackage{multirow}%
\usepackage{amsmath,amssymb,amsfonts}%
\usepackage{amsthm}%
\usepackage{mathrsfs}%
\usepackage[title]{appendix}%
\usepackage{textcomp}%
\usepackage{manyfoot}%
\usepackage{booktabs}%
\usepackage{algorithm}%
\usepackage{algorithmicx}%
\usepackage{algpseudocode}%
\usepackage{listings}%

\usepackage{bbm}
\usepackage[table]{xcolor}

\theoremstyle{thmstyleone}%
\newtheorem{theorem}{Theorem}
\theoremstyle{thmstyletwo}%

\theoremstyle{thmstylethree}%

\newtheorem{assumption}{Assumption}
\newtheorem{assumptionalt}{Assumption}[assumption]

\newenvironment{assumptionp}[1]{
  \renewcommand\theassumptionalt{#1}
  \assumptionalt
}{\endassumptionalt}

\begin{document}

\title[Article Title]{A generalized approach to label shift: the Conditional Probability Shift Model}

\author*[1,2]{\fnm{Paweł} \sur{Teisseyre}}\email{teisseyrep@ipipan.waw.pl}
\equalcont{These authors contributed equally to this work.}

\author[1,2]{\fnm{Jan} \sur{Mielniczuk}}\email{miel@ipipan.waw.pl}
\equalcont{These authors contributed equally to this work.}

\affil*[1]{\orgdiv{Institute of Computer Science}, \orgname{Polish Academy of Sciences}, \orgaddress{\street{Jana Kazimierza 5}, \city{Warsaw}, \postcode{01-248}, \country{Poland}}}

\affil[2]{\orgdiv{Faculty of Mathematics and Information Science}, \orgname{Warsaw University of Technology}, \orgaddress{\street{Koszykowa 75}, \city{Warsaw}, \postcode{00-662}, \country{Poland}}}


\abstract{
In many practical applications of machine learning, a discrepancy often arises between a source distribution from which labeled training examples are drawn and a target distribution for which only unlabeled data is observed. Traditionally, two main scenarios have been considered to address this issue: covariate shift (CS), where only the marginal distribution of features changes, and label shift (LS), which involves a change in the class variable's prior distribution. However, these frameworks do not encompass all forms of distributional shift.
This paper introduces a new setting, Conditional Probability Shift (CPS), which captures the case when the conditional distribution of the class variable given some specific  features changes while the distribution of remaining features given the specific features and the class is preserved. For this scenario we present the Conditional Probability Shift Model (CPSM) based on modeling the class variable's conditional probabilities using multinomial regression.
Since the class variable is not observed for the target data, the parameters of the multinomial model for its distribution are estimated using the Expectation-Maximization algorithm. The proposed method is generic and  can be combined with any probabilistic classifier.
The effectiveness of CPSM is demonstrated through experiments on synthetic datasets and a case study using the MIMIC medical database, revealing its superior balanced classification accuracy on the target data compared to existing methods, particularly in situations situations of conditional distribution shift and no apriori distribution shift, which are not detected by LS-based methods.
}




\maketitle

\section{Introduction}
\subsection{Motivation}
In many real world applications aiming at classification of objects, the source distribution $P$, from which we sample labeled training examples, with a label denoting a class variable,  differs from the target distribution $Q$, where only   unlabeled data  is observed \cite{Zadrozny2004, Davidetal2010, WIDLS2021}.
This situation occurs, for example, in medical applications, where, models trained on patients from one hospital do not necessarily adequately  generalize to new data from other hospitals \cite{Zechetal2018, StackeEilertsenUngerLundstrom2021}.
The shift in the distribution of diseases and/or patient parameters may be related, among other factors, to socio-economic differences in the studied populations or outbreak of an epidemic \cite{Rolandaetal2022}. Also, in banking, distribution of credit repayments and defaults may change in time \cite{CreditRiskDA}.

In general, it is not possible to make inferences for the target data based on information obtained solely from the source dataset \cite{Davidetal2012}, and further assumptions are needed to enable training the classifier. In prior works two scenarios were considered: covariate shift (CS) and label shift (LS).
Under the covariate shift (CS) scheme \cite{HuangGrettonBorgwardtScholkopfSmola2006, BickelBrucknerScheffer2009}, we assume that only the marginal distribution of features changes, and the posterior distribution remains unchanged. 
Label shift (LS) \cite{LiptonWangSmola2018, SaerensLatinneDecaestecker2002, AlexandariKundajeShrikumar2020, AzizzadenesheliLiuYangAnandkumar2019, Gargetal2020} stipulates that the prior distribution of the class variable changes, while the conditional distribution of features given labels remains the same for the source and target datasets.
LS aligns with the anti-causal setting in which
the class variable determines the features, e.g. in medical applications where symptoms  are caused by disease \cite{Scholkopfetal2012, LiptonWangSmola2018}. The LS assumption indicates that the distribution of symptoms in the presence of the disease is the same in the source and target data.

Obviously, the above two scenarios  do not cover all possibilities of distribution shifts. 
Recently a novel  SJS (Sparse Joint Shift) assumption \cite{ChenZahariaZou2024, Tasche2024} has been introduced, which is a generalization of label shift and sparse covariate shift. Namely, it is assumed that distribution of class variable and some features change, but the remaining features’ distribution conditional on the shifted features and class variable is fixed.

 SJS encompasses, as a special case   the shift in  the marginal distribution of the chosen features only. This situation, however,  is of minor interest for the classification problem because, as we  show in the following, the class posterior probabilities for source and the target datasets  match in this case. Much more interesting and demanding is the shift of the conditional distributions of the class variable given some features. Motivated by this, in this paper we introduce a new setting called CPS (conditional probability shift), which covers the most interesting case within the SJS framework.

In particular, we can observe a shift in the conditional distribution  of the class variable given some features, while its   marginal (prior) distribution  is preserved, which means that there is no label shift.
To stress this important point, we discuss  the following example (see Figure \ref{fig:motivation_example}) in which the class variable is the occurrence of a disease, and additionally we consider the age variable with two categories: '$\leq $40 years' and '$>$40 years'. For the source data, the probability of succumbing to the  disease in the group of older people (equal $0.4$) is twice as high as in the group of younger people (equal $0.2$). In the target data, on the other hand, the probability increases
in the older group to $0.5$, while in the younger group it equals $0.1$. Therefore, we observe a change in the conditional distribution of the disease, while its prior distribution, in the case where both ages of the group are equally likely,  remains the same for the source and target data and is equal $0.3$. The above situation may occur when the target data refer to a population in which disease prevention for the elderly is at a low level compared to the source population.

\begin{figure}
\centering
    \begin{tabular}{c c}
      \includegraphics[width=0.4\textwidth]{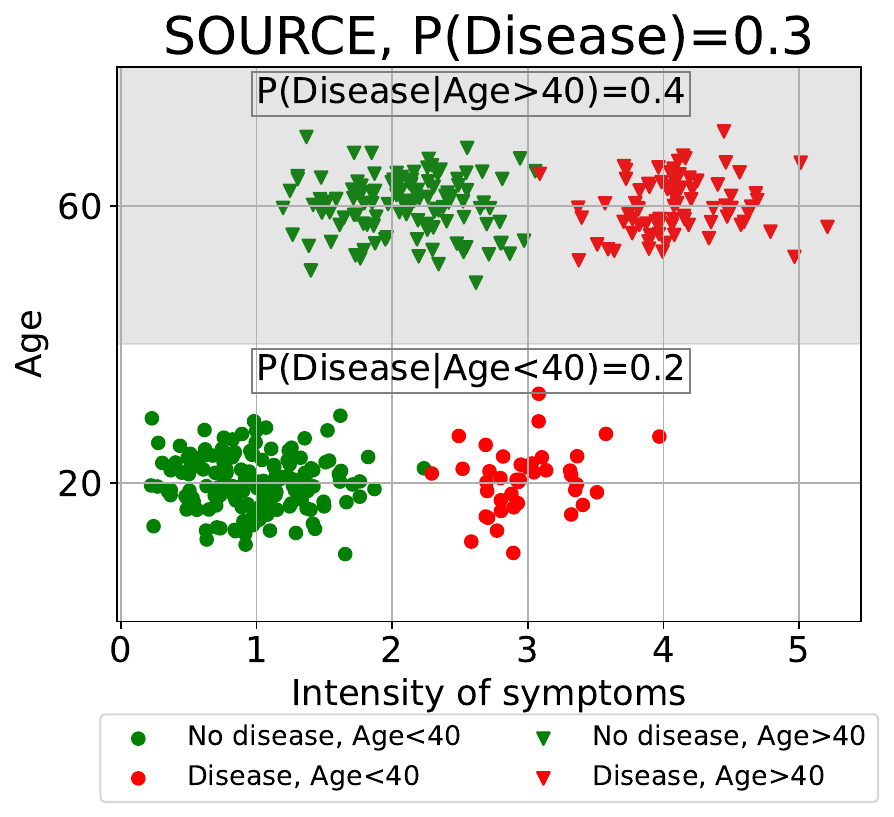} &
      \includegraphics[width=0.4\textwidth]{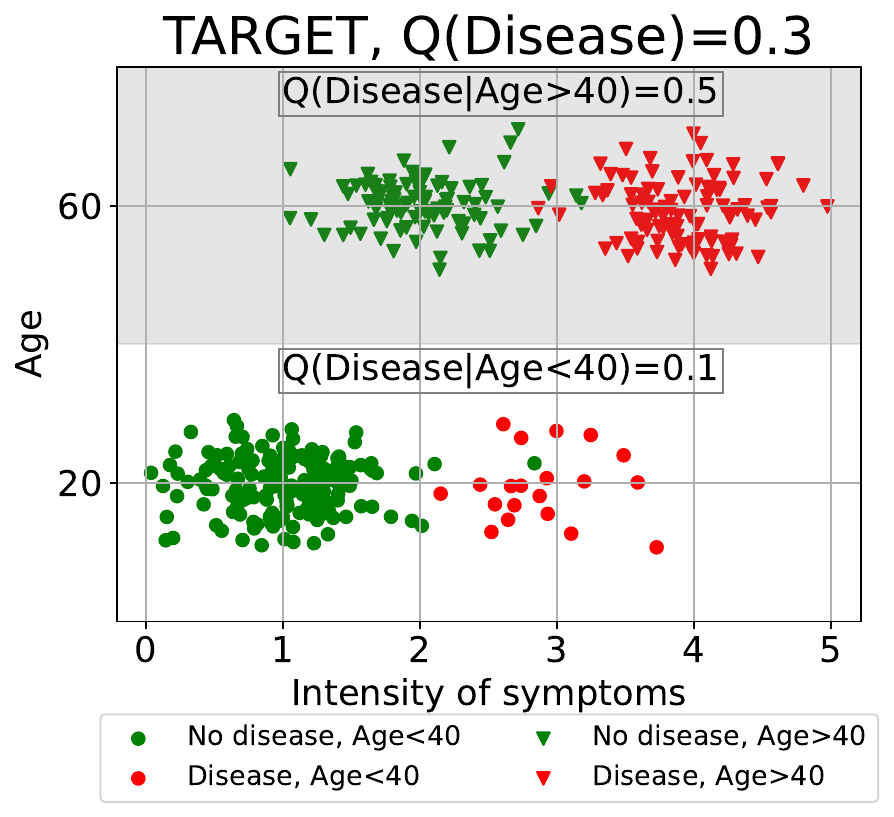}  \\
    \includegraphics[width=0.4\textwidth]{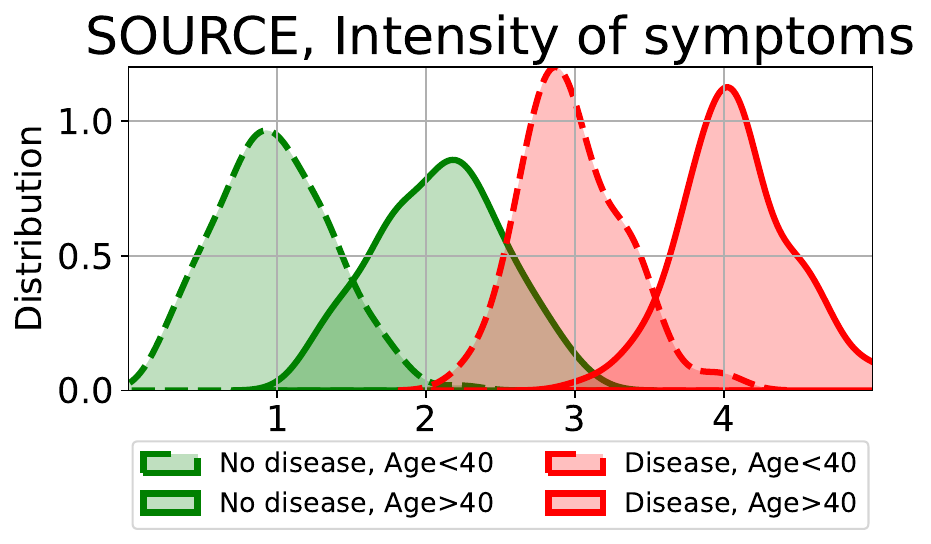} &
      \includegraphics[width=0.4\textwidth]{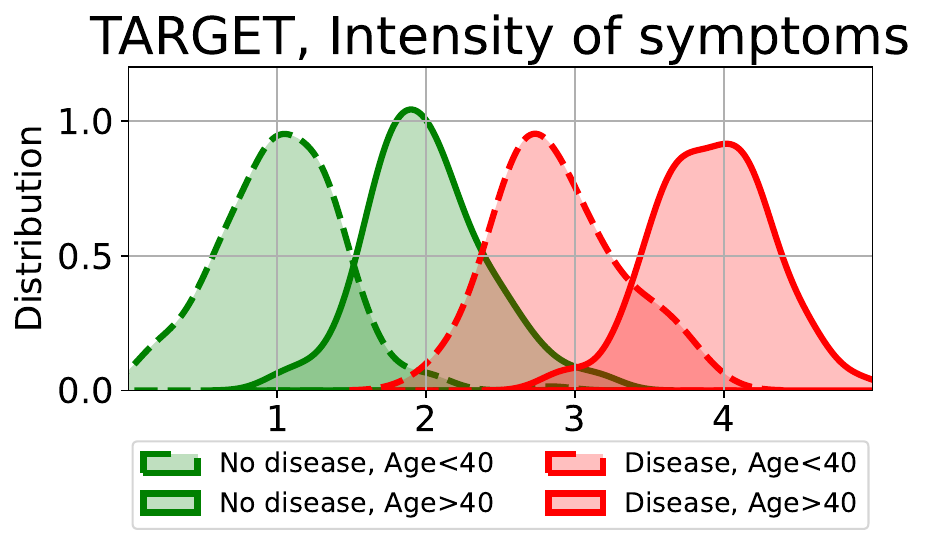}  \\
      \end{tabular}
    \caption{Example of a shift in the conditional distribution of a class variable denoting the occurrence of a disease given  the age (1st row) : $P(Disease|Age)\neq Q(Disease|Age)$. The prior distribution of the disease is the same for the target and source distributions $P(Disease)=Q(Disease)$ (no label shift). The  empirical conditional distribution of the feature (intensity of symptoms) given disease and age are approximately the same for source and target data up to estimation error (2nd row).}
    \label{fig:motivation_example}
\end{figure}

\subsection{Contribution}
In this work, we propose a new model, called CPSM (Conditional Probability Shift Model) that relies on the CPS assumption. The main idea of the method is to model the conditional probability of a class variable given some of the features using multinomial models for the source and target data. Estimating the model parameters for the target data is  challanging because we do not observe the class variable. Therefore, we use the EM algorithm, which allows us to iteratively estimate the parameters of the multinomial model as well as the posterior probability of the class on the target data.

Experiments conducted on artificial data as well as a case study on the medical MIMIC database show that the method provides higher classification accuracy than competing methods dedicated to LS and SJS schemes.
In particular, there is a significant advantage over methods dedicated to the LS scheme (such as BBSC or MLLS described below) in the case when the conditional distribution  is shifted but the prior probability of the target variable does not change.

\subsection{Related work}
The shift of distributions between source and target data is an important problem in machine learning and therefore in recent years much attention has been paid to detecting such a situation \cite{Rabanser2019, Kulinski2020} and adapting models based on the source data \cite{Cortes2008, Sugiyama2007, LiptonWangSmola2018}.

The CS scheme involving a shift in feature probability distribution is less conceptually demanding, since the feature vector is observed for both the source and target data and posterior probabilities do  not change. The dominant approach is based on weighted risk minimization \cite{Sugiyama2007}, where the weights are usually determined using matching techniques, such as Kernel Mean Matching \cite{HuangGrettonBorgwardtScholkopfSmola2006}.

Label Shift (LS) and  a more general sparse joint shift (SJS) are the two most closely related frameworks to the scenario studied here.
Under LS assumption, there are two dominant approaches: BBSC (Black-Box Shift Correction) \cite{LiptonWangSmola2018} and MLLS (Maximum Likelihood Label Shift) \cite{SaerensLatinneDecaestecker2002}.
BBSC method introduces  importance weights that are estimated using the confusion matrix.   The weights are then  incorporated into the risk function, which is optimized on the source data to train a final classifier. 
Regularized Likelihood Label Shift (RLLS) \cite{AzizzadenesheliLiuYangAnandkumar2019} is a variant of BBSC which introduces novel regularization of weighted likelihood to compensate for the high approximation  error of the importance weights.
The approach based on weighted risk function is also used in other methods, such as the method called ExTRA \cite{MaityYurochkinBanerjeeSun2023}, where the weights are estimated using exponential tilt model and distribution matching technique.  The ExTRA method allows for covariate and  label shift to occur at the same time.
The above methods require retraining; the first model is trained on the original source dataset, and the second model is trained on the dataset with assigned sample weights. This can be computationally expensive, especially for large-scale data.
On the other hand, the MLLS method \cite{SaerensLatinneDecaestecker2002} avoids this drawback by calibrating class posterior probabilities. The method exploits the fact that the class posterior probability for the target data can be written as a transformation of the class posterior probability for the source data, with the scaling factor depending on the unknown target class prior. Since the class variable is not observable for the target data, the EM procedure is used, which allows iteratively estimating the posterior and  prior probabilities for the target data. The method proposed in this paper is based on a similar idea, but the main difference is that instead of estimating the class prior, we estimate the parameters of a multinomial model associated with the conditional probabilities. When combined with deep neural networks, the MLLS method can be further improved by using a calibration heuristic called Bias-Corrected Temperature Scaling (BCTS) \cite{AlexandariKundajeShrikumar2020}. 
In \cite{Gargetal2020}, the BBSC and MLLS are analyzed theoretically, in particular consistency conditions for
MLLS, including calibration of the classifier and a confusion matrix invertibility, are provided. LS problem is recently  modified  to allow for partial observability of the source data, e.g. in \cite{Nakajima23} it is assumed that it is Positive-Unlabeled (PU).

The SJS scheme, which generalizes LS, was proposed very recently and therefore not as many algorithms have been developed yet as for LS.
SEES \cite{ChenZahariaZou2024} is, to the best of our knowledge,  the sole method developed under SJS assumption. Discussions on the theoretical properties of the SJS scheme can be found in \cite{Tasche2024}.
SEES determines sample weights by minimizing distance between the induced feature density and the target feature density.

Finally, let's mention that there are other interesting issues related to the distribution shift.
For example, domain adaptation is a research area concerned with a   problem for which data availability  is similar to   that  discussed here. Namely, for  Unsupervised Domain Adaptation (UDA)  the source data consists of a set of labeled data and unlabeled data, the latter possibly coming from  different distribution than the labeled data. 
The task then consists of building a classifier based {\it jointly} on  those two types of source  data,   which is meant to classify new observations  following a target distribution, usually  equal to distribution of unlabeled source data. For recent advances of UDA see e.g. \cite{Liu2022}.
Another related problem is out-of-distribution (OOD) detection, where the target  unlabeled observations follow  a mixture of the source distribution and an unknown  distribution of outliers. The goal is to decide whether the given target observation is drawn from the source distribution or from the distribution of outliers, based on a classifier trained on the source data \cite{OCC, OODSurvey}.

\section{Background}

We consider a $K$-class classification problem, where each instance is described by a feature vector $(x,z)\in \mathcal{X}\times \mathcal{Z}$ and class variable (label) $y\in \mathcal{Y}$, where $\mathcal{Y}=\{1,\ldots,K\}$ is the output space.
Feature vector $z$ is considered separately  from the remaining  feature vector $x$ since  it  will determine the shift of the conditional distribution of labels as a conditioning variable.
Furthermore, let $P$ and $Q$ denote respectively the source and the target probability distributions on $\mathcal{X}\times\mathcal{Z}\times \mathcal{Y}$ and let $p(x,z,y)$ and $q(x,z,y)$ denote the corresponding probability densities or probability mass functions. It is assumed that the source and the target distributions differ.
Importantly,  for the  source dataset we observe $x,z,y$, whereas for the target dataset, we only observe feature vector $(x,z)$, whereas class $y$ is not observed. The goal is to predict $y$ in the target dataset.

There are various ways to approach distribution shift between
a source data distribution $P$ and a target data distribution $Q$ \cite{Zadrozny2004}. 
Without any assumptions  concerning the relation between $P$ and $Q$, the task is clearly impossible. The existing works focus mostly on covariate shift and label shift.
The covariate shift (CS) \cite{HuangGrettonBorgwardtScholkopfSmola2006, BickelBrucknerScheffer2009}  means that the difference in distributions of $(x,z,y)$ is due to a difference in  distribution of covariates:
$p(x,z)\neq q(x,z)$, whereas the conditional  posterior distribution is preserved, i.e. $p(y|x,z)=q(y|x,z)$.
On the other hand,  for label shift (LS) \cite{LiptonWangSmola2018, SaerensLatinneDecaestecker2002, AlexandariKundajeShrikumar2020, AzizzadenesheliLiuYangAnandkumar2019}  it is assumed that the difference in distributions results from  the difference in prior distributions: $p(y)\neq q(y)$, but the conditional distribution of the feature vector given class variable is preserved, i.e. $p(x,z|y)=q(x,z|y)$.
Recently, in \cite{ChenZahariaZou2024, Tasche2024}, more general assumption has been introduced, called Sparse Joint Shift (SJS). Under SJS, we have that $p(x,z,y)\neq q(x,z,y)$ is due to the fact that the joint distribution of label $y$ and part of the features $z$ is shifted $p(y,z)\neq q(y,z)$, but the conditional distribution of $x$ given $y$ and $z$ is preserved, i.e.,
$
p(x|y,z) = q(x|y,z).
$
 The considered extension is important from  practical point of view: consider  e.g. $y$ being a certain disease, $x$ its symptoms and $z$ an age category. Assuming that the data distribution changes in time, condition $p(y,z)\neq q(y,z)$ means that  change of occurrence of disease differs in  various  age categories, but the assumption  $p(x|y,z) = q(x|y,z)$ signifies that characteristics of symptoms is unchanged given illness and age. Figure \ref{fig:motivation_example} (bottom row) illustrates this situation. 
The distribution of symptom intensity (variable x) conditional on age (z) and disease occurence (y) remains approximately the same for the source and target data; slight differences result only from the fact that these are empirical distributions, estimated from the data.
 We stress that even in anti-causal setting it may happen that $p(x|y,z) = q(x|y,z)$ but $p(x|y) \neq  q(x|y)$ due to $p(z)\neq q(z)$.

\begin{table}[ht!]
\centering
\caption{Types of shift in joint distributions between source and target domains. CPS stands for the Conditional Probability Shift scheme considered in this paper. MS denotes marginal shift of $z$.}
\label{table:Types}
\begin{tabular}{lll}
\toprule
\textbf{Distribution shift}    & \textbf{Source domain} & \textbf{Target domain} \\
    & $p(x,z,y)$ & $q(x,z,y)$ \\
\midrule
Covariate shift \cite{HuangGrettonBorgwardtScholkopfSmola2006, BickelBrucknerScheffer2009} & $\mathbf{p(x,z)}p(y|x,z)$ & $\mathbf{q(x,z)}p(y|x,z)$\\ 
Label shift \cite{LiptonWangSmola2018, SaerensLatinneDecaestecker2002, AlexandariKundajeShrikumar2020, AzizzadenesheliLiuYangAnandkumar2019, Gargetal2020} & $\mathbf{p(y)}p(x,z|y)$ & $\mathbf{q(y)}p(x,z|y)$\\
Sparse Joint Shift \cite{ChenZahariaZou2024, Tasche2024}: & $\mathbf{p(z,y)}p(x|z,y)$ & $\mathbf{q(z,y)}p(x|z,y)$\\
-case 1 CPS  & $\mathbf{p(y|z)}p(z)p(x|z,y)$ & $\mathbf{q(y|z)}p(z)p(x|z,y)$\\
-case 2 CPS + MS  & $\mathbf{p(y|z)p(z)}p(x|z,y)$ & $\mathbf{q(y|z)q(z)}p(x|z,y)$\\
-case 3 MS  & $p(y|z)\mathbf{p(z)}p(x|z,y)$ & $p(y|z)\mathbf{q(z)}p(x|z,y)$\\
\bottomrule
\end{tabular}
\end{table}

Observe that under  SJS assumption, in view of $p(y,z)=p(y|z)p(z)$, the shift of $(y,z)$ distribution  can be associated with two  possible shifts. First, it can happen that the conditional distributions do not change, i.e. $p(y|z)=q(y|z)$, and the shift $p(z,y)$ results only from the  fact that $p(z)\neq q(z)$. However, provided $p(x|y,z)=q(x|y,z)$, such situation does not cause any new problems for classification, because then $q(y|x,z)=p(y|x,z)$. This follows e.g.  from
Theorem \ref{Thm1} below.
The second case when $p(y|z)\neq q(y|z)$ is much more interesting. So, in this paper we introduce a novel shift assumption, which  extracts  the most important case  from  SJS  assumption,  corresponding  to the change in conditional distributions.
\begin{assumptionp}{(CPS)}
Under  Conditional Probability Shift (CPS) assumption,  the difference in distributions   
$p(x,z,y)\neq q(x,z,y)$ is due to  the fact that  the conditional  probability $q(y|z)$ of label $y$ given  features $z$ is shifted : $p(y|z)\neq q(y|z)$, but the conditional distribution of the remaining features $x$ given $y$ and $z$ is preserved, i.e.
\begin{equation}
\label{CPS}
p(x|y,z) = q(x|y,z).
\end{equation}
\end{assumptionp}
Table \ref{table:Types} summarizes the different types of shift in probability distributions. Bold lettering corresponds to quantities which change for $P$ and $Q$.
We stress that  (\ref{CPS}) is assumed in both CPS and SJS introduced previously. The difference is that in SJS the shift is associated with the distribution $(y,z)$ while in CPS specifically with the conditional distribution $y$ given $z$. The CPS assumption excludes the  case of marginal interest  from the SJS assumption (case 3 in Table \ref{table:Types}) and covers more interesting situations (cases 1 and 2 in Table \ref{table:Types}). In this paper,  under  CPS setting,  the inference  is based on modelling   conditional distributions $p(y|z)$ and $q(y|z)$.
We will show in the following that (\ref{CPS}) is valuable side information which allows to perform classification despite  occurrence of  distributional  shift.

Let us also note that, the term \textit{sparse}, used in SJS,  referred to the fact that the $z$-dimension is significantly smaller than the $x$-dimension. Unlike in SJS, in our approach, this assumption is not needed and the $z$-dimension with respect to the $x$-dimension can be arbitrary, although in practice it is usually low.
Obviously, when $z$ is null vector, SJS assumption coincides with LS.


The following theorem will be crucial to our method. It shows the relationship between the posterior probabilities for the source and target domains. It is stated in \cite{Tasche2024} using different formalism and assumptions, here we state and prove it in probabilistic setting. It  generalizes the result for the LS scheme, provided in \cite{SaerensLatinneDecaestecker2002}.
\begin{theorem}
\label{Thm1}
Assume that (\ref{CPS}) holds. Then we have for $k=1,\ldots,K$
\[
q(y=k|x,z)=\frac{p(y=k|x,z)\frac{q(y=k|z)}{p(y=k|z)}}{\sum_{l=1}^{K}p(y=l|x,z)\frac{q(y=l|z)}{p(y=l|z)}}.
\]
\end{theorem}
\begin{proof}
Let us denote $r(x,z):=\frac{p(x|z)}{q(x|z)}$. Using Bayes' theorem and assumption (\ref{CPS}) we get
\begin{align*}
q(y=k|x,z)=&\frac{q(x|z,y=k)q(y=k|z)}{q(x|z)} =  \frac{p(x|z,y=k)q(y=k|z)}{q(x|z)}  \\
=& \frac{p(y=k|x,z)p(x|z)q(y=k|z)}{p(y=k|z)q(x|z)} = p(y=k|x,z)\cdot r(x,z)\cdot \frac{q(y=k|z)}{p(y=k|z)}.
\end{align*}
Then using the fact that $\sum_{k=1}^{K}q(y=k|x,z)=1$, summation  over $k$ of the above expression yields
\[
r(x,z)=\left[\sum_{l=1}^{K}p(y=l|x,z)\frac{q(y=l|z)}{p(y=l|z)}\right]^{-1},
\]
which ends the proof.
\end{proof}
The theorem justifies why considering the CPS scheme is important within the general sparse joint shift situation.
It indicates that provided (\ref{CPS}) holds, in order to recover the class posterior $q(y|x,z)$, we need only to estimate  quantities $p(y|x,z)$, $p(y|z)$ and $q(y|z)$ but neither $p(z)$ or $q(z)$. This is then  used to classify observations on the target set based on the standard Bayes rule:
\begin{equation}
\label{BayesRule}
\hat{y}=\arg\max_{k} q(y=k|x,z)    
\end{equation}
or its variants, taking into account imbalance of the classes.



\section{Modeling shifted conditional distribution of labels: EM approach}
We propose to model the conditional probabilities $p(y=k|z)$ and $q(y=k|z)$ using a multinomial regression models. 
Estimation of the first probability $p(y=k|z)$ is straightforward  using multinomial regression because we observe both $z$ and $y$ for the source data. However, it is not evident how to estimate $q(y=k|z)$,   as class indicator $y$ is not observed for target data.
We consider a family of  soft-max functions indexed by $\theta$:
\begin{equation}
\label{CPSM}
q_{\theta}(y=k|z)=\frac{\exp(\theta_{k,0} + z^{T}\theta_{k})}{1+\sum_{l=1}^{K-1}\exp(\theta_{l,0} + z^{T}\theta_{l})},\quad 1\leq k< K,
\end{equation}
where $\theta=(\theta_{1,0},\ldots,\theta_{K-1,0},\theta_1,\ldots,\theta_{K-1})^{T}$ is a vector of parameters. Additionally, we have $q_{\theta}(y=K|z)=(1+\sum_{l=1}^{K-1}\exp(\theta_{l0} + z^{T}\theta_{l}))^{-1}$. 
Assume that the true probability $q(y=k|z) = q_{\theta^{*}}(y=k|z)$, for some unknown ground-truth parameter $\theta^{*}$.  
In the following, we  propose an estimation method of  $\theta^*$, which in turn will allow us to estimate $q(y=k|x,z)$ using Theorem \ref{Thm1}.
Define indicator of the $k$-th class as $y_k=I(y=k)$.
Using $q(x,z,y)=q(z)q(y|z)q(x|y,z)$   the log-likelihood function for a single observation in target population can be written as
\begin{align}
\label{LogLik}
l(\theta;x,y,z)=& \log\prod_{k=1}^{K}q_{\theta}(x,z,y=k)^{y_k}=\sum_{k=1}^{K}y_k\log\left[q_{\theta}(x,z,y=k)\right]\\
=& \sum_{k=1}^{K}y_k\log[q(z)] + \sum_{k=1}^{K}y_k\log[q_{\theta}(y=k|z)] + \sum_{k=1}^{K}y_k\log[q(x|z,y=k)] \nonumber\\
=& \log[q(z)] +\sum_{k=1}^{K-1}y_k[\theta_{k,0}+ z^{T}\theta_k]-\log\left[1+\sum_{l=1}^{K-1}\exp(\theta_{l,0}+z^{T}\theta_l)\right]\nonumber \\
+&  \sum_{k=1}^{K}y_k\log q(x|z,y=k).\nonumber
\end{align}
Note that the last term in the expression above does not depend on $\theta$ due to CPS assumption \ref{CPS}.
Since $y_{k}$ is not observed for the target data, we can treat it as a latent variable and maximize the above function using the EM algorithm.
Let us denote by $\mathcal{L}(\theta,\mathcal{D}_t)=\sum_{(x,z,y)\in \mathcal{D}_t}l(\theta;x,y,z)$ the (unobservable) log-likelihood function for all observations in the target dataset $\mathcal{D}_t$.
\begin{enumerate}
    \item {\bf Expectation step (E step):} determine  the expected value of the log likelihood function, with respect to the current conditional distribution of $y$ given $x$ and $z$ corresponding to  the current estimate of the parameter $\widehat{\theta}^{(t)}$:
\begin{equation*}
Q(\theta|\widehat{\theta}^{(t)}) = \mathbb{E}_{y\sim q_{\widehat{\theta}^{(t)}}(y|x,z)}\mathcal{L}(\theta;\mathcal{D}_t)
= \sum_{(x,z)\in\mathcal{D}_t}\sum_{k=1}^{K}q_{\widehat{\theta}^{(t)}}(y=k|x,z)\log\left[q_{\theta}(x,z,y=k)\right],   
\end{equation*}
where  $q_{\widehat{\theta}^{(t)}}$ is given  by the formula
\[
q_{\widehat{\theta}^{(t)}}(y=k|x,z):=\frac{p(y=k|x,z)\frac{q_{\widehat{\theta}^{(t-1)}}(y=k|z)}{p(y=k|z)}}{\sum_{l=1}^{K}p(y=l|x,z)\frac{q_{\widehat{\theta}^{(t-1)}}(y=l|z)}{p(y=l|z)}}
\]
which follows from Theorem \ref{Thm1}. Here, we assume that $p(y=k|x,z)$ and $p(y=k|z)$ are known. Obviously, in practice they have to be estimated. In the proposed method, we first estimate $p(y=k|x,z)$ and $p(y=k|z)$ using source data, and then apply EM procedure using target data.
\item 
{\bf Maximization step (M-step):} find the parameters that maximize
\begin{align*}
&\widehat{\theta}^{(t+1)} = \arg\max_{\theta} Q(\theta|\widehat{\theta}^{(t)}) \\
&=\arg\max_{\theta}\sum_{(x,z)\in\mathcal{D}_t}\left[\sum_{k=1}^{K-1}q_{\widehat{\theta}^{(t)}}(y=k|x,z)[\theta_{k,0}+ z^{T}\theta_k]-\log[1+\sum_{l=1}^{K-1}\exp(\theta_{l,0}+z^{T}\theta_l)]\right],
\end{align*}
\end{enumerate}
where the second equality follows from the reasoning in (\ref{LogLik}).
Thus, in M-step instead of maximization of log-likelihood which is unobservable, we optimize its
expected value with respect to $y$, given current value of the estimated parameter.
The above optimization problem is concave and thus can be solved using standard gradient methods such as SGD or ADAM.
The following theorem states  the main property of EM algorithm, namely  that at each step of the EM procedure  the value of the marginal log-likelihood function 
for the observed data $(x,z)$ does not decrease.  In order to facilitate reading, we give the proof for  our setting  when $y$ is latent variable.
\begin{theorem}
Let $l_{\textrm{obs}}(\theta;x,z)=\log q_{\theta}(x,z) = \log\sum_{k}q(z)q_{\theta}(y=k|z)q(x|y=k,z)$ and $\mathcal{L}_{\textrm{obs}}(\theta,D_t)=\sum_{(x,z)\in\mathcal{D}_t}l_{\textrm{obs}}(\theta;x,z)$ The following inequality holds 
\[
\mathcal{L}_{\textrm{obs}}(\widehat{\theta}^{(t+1)},\mathcal{D}_t)\geq \mathcal{L}_{\textrm{obs}}(\widehat{\theta}^{(t)},\mathcal{D}_t)
\]
\end{theorem}
\begin{proof}
First note that
$
l_{\textrm{obs}}(\theta;x,z) =\log q_{\theta}(x,y,z) - \log q_{\theta}(y|x,z).
$
Calculating  the expected value $\mathbb{E}_{y\sim q_{\widehat{\theta}^{(t)}}(y|x,z)}$ of  both sides and summing over all $(x,z)\in\mathcal{D}_t$ we get
\[
\mathcal{L}_{\textrm{obs}}(\theta,\mathcal{D}_t)=Q(\theta|\widehat{\theta}^{(t)})  - \sum_{(x,z)\in\mathcal{D}_t}\mathbb{E}_{y\sim q_{\widehat{\theta}^{(t)}}(y|x,z)}\log q_{\theta}(y|x,z).
\]
Using the above equality, the assertion follows from
\begin{align*}
&\mathcal{L}_{\textrm{obs}}(\widehat{\theta}^{(t+1)},\mathcal{D}_t)-\mathcal{L}_{\textrm{obs}}(\widehat{\theta}^{(t)},\mathcal{D}_t)=
Q(\widehat{\theta}^{(t+1)}|\widehat{\theta}^{(t)})-Q(\widehat{\theta}^{(t)}|\widehat{\theta}^{(t)}) \\
&+ \sum_{(x,z)\in\mathcal{D}_t}\mathbb{E}_{y\sim q_{\widehat{\theta}^{(t)}}(y|x,z)}\log q_{\widehat{\theta}^{(t)}}(y|x,z) - \sum_{(x,z)\in\mathcal{D}_t}\mathbb{E}_{y\sim q_{\widehat{\theta}^{(t)}}(y|x,z)}\log q_{\widehat{\theta}^{(t+1)}}(y|x,z)\geq 0,
\end{align*}
where the non-negativity of the first term follows from the fact that $\widehat{\theta}^{(t+1)}$ maximizes $Q(\theta|\widehat{\theta}^{(t)})$ (M-step), and the non-negativity of the second term follows from the Gibbs inequality $\sum_{k}p_k\log p_k\geq \sum_{k}p_k\log q_k$, for discrete probability distributions $p_k,q_k$. 

\end{proof}
If optimisation in M-step is performed over a a bounded set, it follows from Theorem 2 in \cite{Wu1983} that all limit points $\theta^*$ of the  sequence  $(\widehat{\theta}^{(t)})$ are stationary points i.e. $ \nabla \mathcal{L}_{\textrm{obs}}(\theta^*,\mathcal{D}_t)=0$.\\
As discussed above, application of the EM procedure is preceded by the estimation of $p(y=1|z)$ and $p(y=1|z,x)$ based on the source data. We assume that $p(y=1|z)$ is modeled using multinomial logistic regression (MLR), which is natural since $z$ is usually a low-dimensional vector containing only a few variables. The question arises what model we should use to estimate $p(y=1|z,x)$? Of course, allowing for incorrect specification, we can use any classification model which yields estimates of posterior probabilities, but the following Theorem shows that, with the additional assumption that the distribution of $x$ given features $z$ and label $y$ is Gaussian, the adequate choice is also MLR. In view of the theorem, we use MLR as a basic model in our experiments to estimate $p(y=1|x,z)$, although other models such as deep neural networks are also applied. 

\begin{theorem}
\label{thm_normal}
Consider class variable $y\in\{1,\ldots,K\}$, $z\in R^d$ and assume  that 
$
p(y=1|z)~=\exp(\omega^*_{k0}+z^{T}\omega^*_{k})/[1+\sum_{l=1}^{K-1}\exp(\omega^*_{l0}+z^{T}\omega^*_{l})],
$
for $1\leq k\leq K-1$, where $\omega^*_{k0}, \omega^*_{k}$ are unknown ground-truth parameters. Moreover, assume that conditional distribution of $x$ given $y,z$ is $\mathcal{N}(Mz+a_ky_k,I)$, where $ M\in R^{p\times d}, a_k\in R^{p}$ and $y_k=I(y=k)$. Then the posterior probability of $y$, given $x,z$ is also described by softmax function
\[
P(y=k|x,z)=\frac{\exp[x^{T}(a_k-a_K)+z^{T}(M^{T}(a_K-a_k)+\omega^*_{k})+c]}{1+\sum_{l=1}^{K-1}\exp[x^{T}(a_l-a_K)+z^{T}(M^{T}(a_K-a_l)+\omega^*_{k})+c]}
\]
where constant $c=0.5(a_{K}^{T}a_K-a_k^{T}a_k)+\omega^*_{k0}$.
\end{theorem}
\begin{proof}
Using Bayes' theorem and the normality assumption we get
\begin{align*}
\log\frac{P(y=k|x,z)}{P(y=K|x,z)}=& \log\frac{P(x|y=k,z)}{P(x|y=K,z)} + \log\frac{P(y=k|z)}{P(y=K|z)}  \\
= &-0.5|| x-Mz-a_k||^2  + 0.5 || x-Mz-a_K||^2 + \omega^*_{k0} + z^{T}\omega^*_k \\
= & (x-Mz)^{T}(a_k-a_K)  + z^{T}\omega^*_k + c,
\end{align*}
which after simple algebraic transformations gives the assertion.
\end{proof}



\section{Experiments}
In the experiments, we aim to address the following research questions.
1) What is the accuracy of the proposed CPSM method compared to existing methods dedicated to the LS and SJS scenarios?
2) How accurately do we approximate the posterior probability for the target data using the methods considered?
3) How does the magnitude of the shift in the conditional distributions affect the results?
4) How does the proposed method perform in a situation where we observe a shift in the conditional distributions $q(y|z)\neq p(y|z)$ but there is no change  in the marginal distributions, i.e. $q(y)=p(y)$?
5) How does the number of instances in the data and the class prior probabilities affect the performance of the method?

\subsection{Methods and evaluation}
We compare the method CPSM proposed here \footnote{The source code is available at: \url{https://github.com/teisseyrep/CPSM}} with the following baselines: ExTRA (Exponential Titling Model) \cite{MaityYurochkinBanerjeeSun2023}, SEES (Shift Estimation and Explanation under SJS) \cite{ChenZahariaZou2024}, MLLS (Maximum Likelihood Label Shift) \cite{SaerensLatinneDecaestecker2002}, BBSC (Black-Box  Shift Correction) \cite{LiptonWangSmola2018} and  also the NAIVE method in which a classification model is trained on the source data and applied to the target data without any corrections. Recall, that MLLS and BBSC are methods developed  assuming the LS scenario.

In the case of competing methods, we use implementations made publicly available by the authors. To make the comparison reliable, in all iterative methods (CPSM, MLLS, ExTRA), we set the number of iterations equal to 500. As base classifiers, we use   logistic regression model (for artificial data and MIMIC data) and also  a neural network (for MIMIC data). In the case of logistic regression, we use the implementation from the scikit-learn library~\cite{scikitlearn}, with default values of the hyper-parameters. The neural network consists of 3 layers, each of which has the number of nodes equal to the number of input features. For optimization, we apply the ADAM algorithm, the learning constant is equal to 0.0001, and the batch size is 10. We use the implementation from the PyTorch library \cite{paszke2019pytorch}. We emphasize that all described methods are generic and can be combined with any probabilistic classifier.

In the experiments, we focus on the case of binary classification, for which the classification rule has the form: $q(y=1|x,z)>t$, where $t$ is a threshold. The choice of $t=0.5$ corresponds to Bayes' rule (\ref{BayesRule}) and maximization of classification accuracy. Since in the considered data, classes are usually strongly imbalanced, instead of the accuracy score on the target data, we consider the Balanced Accuracy, the maximization of which results in  the rule $q(y=1|x,z)>q(y=1)$ (see \cite{Menon13}, Section 4), and $q(y=1)$ is estimated by the respective algorithm as an average of $\hat q(y=1|x,z)$ over all samples.
In addition we consider the Approximation Error defined as $|\hat{q}_{\text{method}}(y=1|x,z)-\hat{q}_{\text{oracle}}(y=1|x,z)|$, where $\hat{q}_{\text{method}}(y=1|x,z)$ is estimator  pertaining to the considered method, whereas $\hat{q}_{\text{oracle}}(y=1|x,z)$ is estimator which results from the model trained on target data and assuming knowledge of $y$. The Approximation Errors are averaged over all instances $(x,z)$ in target data.
\subsection{Experiments on synthetic datasets}
For  datasets generated synthetically, we can easily control the values of various parameters such as $p(y)$, $q(y)$ and the relationship between $y$ and $z$, as well as between $x$ and $(y,z)$.
In synthetic dataset 1, we consider a binary  class variable $y\in\{0,1\}$, and binary variables  $z_j\in\{0,1\}$ being 
independent Bernoulli random variables  such that $p(z_j=1)=q(z_j=1)=0.5$ and let
$z=(z_1,z_2,z_3,z_4,z_5)$. In synthetic dataset 2, $z_j\sim N(0,1)$, for $1 \leq j\leq 5$.
Moreover, in both cases, we assume that 
\[
p(y=1|z)=0.05,\quad\quad q(y=1|z)=\sigma(\theta_0+k z^{T}\mathbbm{1}),\quad \mathbbm{1}=(1,\ldots,1)^{T},
\]
where $\sigma:R\to (0,1)$ is logistic sigmoid function. For fixed value $k$, we choose $\theta_0$ to control $q(y=1)$. In total, we have 2 parameters whose values we will change in the experiment: $q(y=1)$ and $k$.
Finally, given $y$ and $z$ we generate $x$ from $10$-dimensional Gaussian distribution $\mathcal{N}(\mu,I)$, where $\mu=(y,z_1,z_2,z_3,z_4,z_5,0,\ldots,0)^T$. This corresponds to the situation described in Theorem \ref{thm_normal}. Note that, given $z$ and $y$, vector $x$ is generated in the same way for the source and target datasets, which is consistent with the assumption (\ref{CPS}). 
Importantly, we can distinguish 4 interesting cases:
\begin{enumerate}
    \item {\bf No shift}: for $p(y=1)=q(y=1)$ and $k=0$ there is no shift in probability distributions.
    \item {\bf LS}: for $p(y=1)\neq q(y=1)$ and $k=0$,  label shift (LS) occurs, which follows from 
    $p(x,z|y) = p(x|z,y)p(z|y)=q(x|z,y)p(z)=q(x|z,y)q(z)=q(x|z,y)$, as $y$ and $z$ are independent wrt $P$ and $Q$ and $p(z)=q(z)$.
    \item {\bf CPS}: for $p(y=1)=q(y=1)$ and $k\neq 0$ there is no label shift (LS), but there is shift in conditional distributions $p(y|z)\neq q(y|z)$.
    \item {\bf CPS + LS}: for $p(y=1)\neq q(y=1)$ and $k\neq 0$ there is label shift (LS), and  there is a shift in conditional distributions $p(y|z)=p(y)\neq q(y|z)$ but $p(x|y,z)=q(x|y,z)$.
\end{enumerate}
In case 1, when there is no distributional shift, as expected, all methods, including the naive method, perform comparably (Figures \ref{fig:artificial1}, \ref{fig:artificial1a}  top left panels for $k=0$). 
In case 2, when the label shift occurs, we observe an increase in the accuracy of the investigated  methods in relation to the naive method (see Figures \ref{fig:artificial1}, \ref{fig:artificial1a}, all panels except top left and $k=0$). This effect becomes more pronounced when $q(y=1)$ increases, which is associated with an increase in the difference between $p(y)$ and $q(y)$.
In cases 3 and 4 ($k\neq 0)$, when the distribution of $y$  given $z$ changes between the source and target datasets, we see a clear advantage of the proposed method. Moreover,  the proposed method and SEES  work better the BBSC and MLLS methods, due to the fact that the former are based on the more general SJS assumption, so they are adapted to detect the change of the conditional distribution of $y$ given $z$.
We also note that CPSM perform similarly for the first panel and $k=5$ (case 3) and the other panels for $k=5$ (case 4) which shows that it is robust to  the change in a marginal distribution of $y$.

Figures \ref{fig:artificial2} and \ref{fig:artificial2}  show how the  varying number of instances  influences the  results  for $k=5$ and $p(y=1)=0.05$ (we assume that the target data have the same number of observations as the source data). This corresponds to case 4 listed above.  As expected, Balanced Accuracy increases with the sample size, while the Approximation Error decreases. Importantly,  CPSM  outperforms the other methods for all sample sizes. Method ExTRA requires more data to reach the error level of the other methods. For the naive method, the approximation  error remains high, regardless of the sample size due to misspecification effects.

\begin{figure}
\centering
    \begin{tabular}{c c}
      \includegraphics[width=0.4\textwidth]{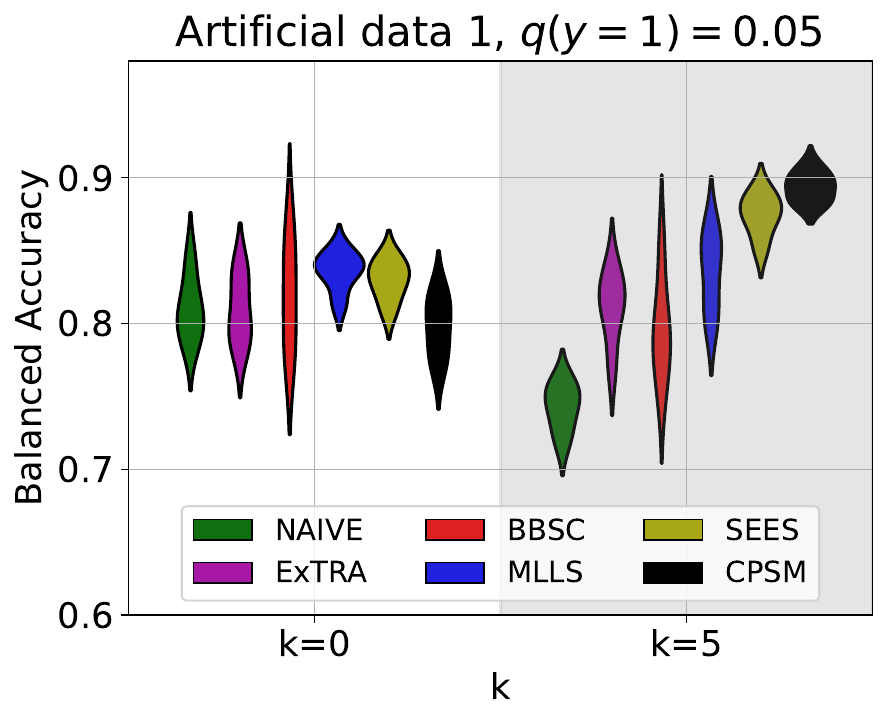} &
      \includegraphics[width=0.4\textwidth]{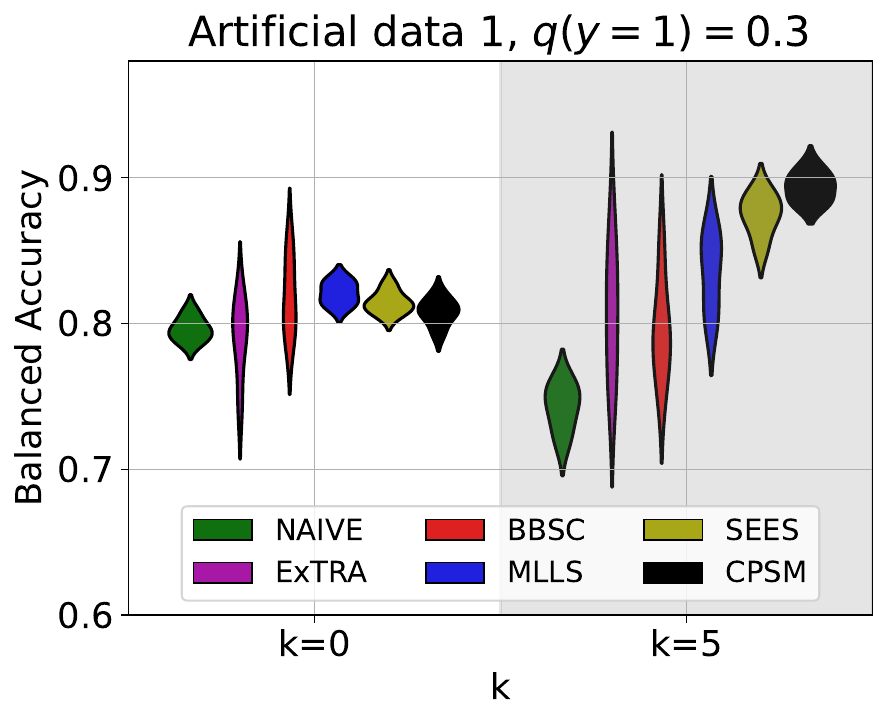}  \\
      \includegraphics[width=0.4\textwidth]{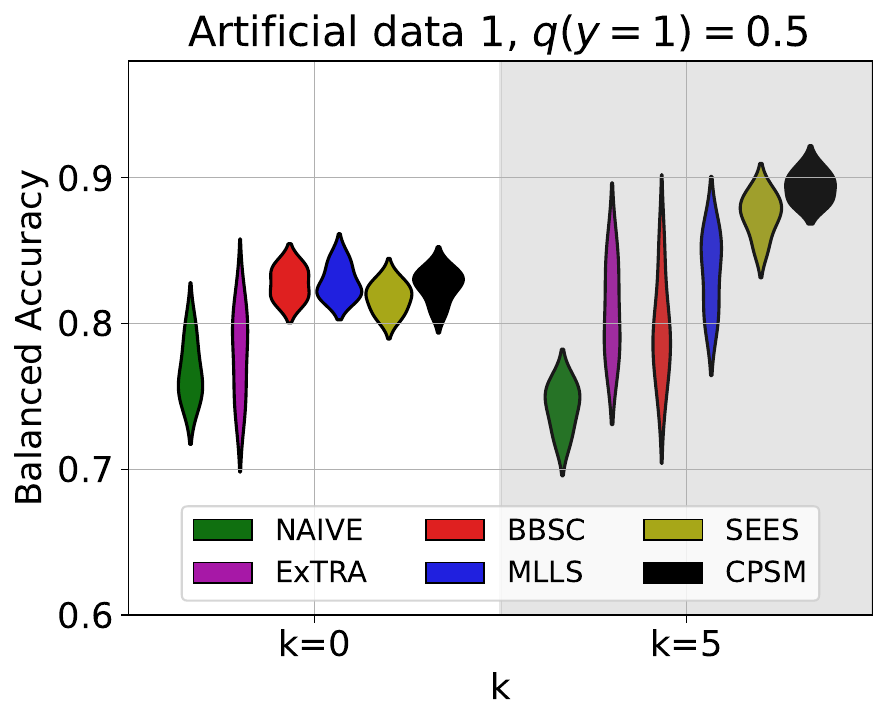} &
      \includegraphics[width=0.4\textwidth]{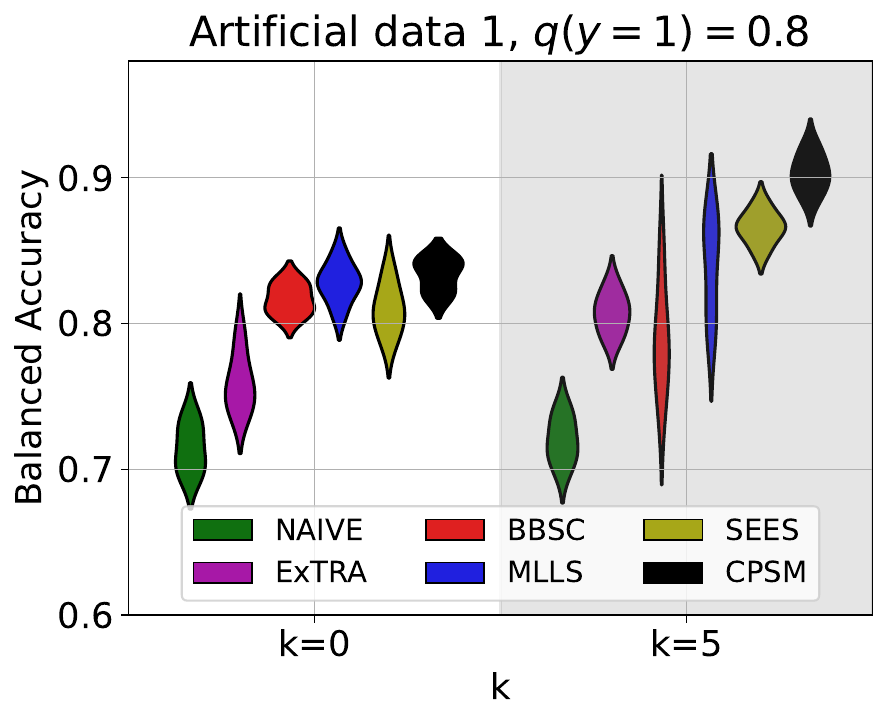}  \\
      \end{tabular}
    \caption{Balanced Accuracy for artificial data 1, for $p(y=1)=0.05$ and varying class prior for target data $q(y=1)=0.05,0.3,0.5,0.8$.}
    \label{fig:artificial1}
\end{figure}

\begin{figure}
\centering
    \begin{tabular}{c c}
      \includegraphics[width=0.4\textwidth]{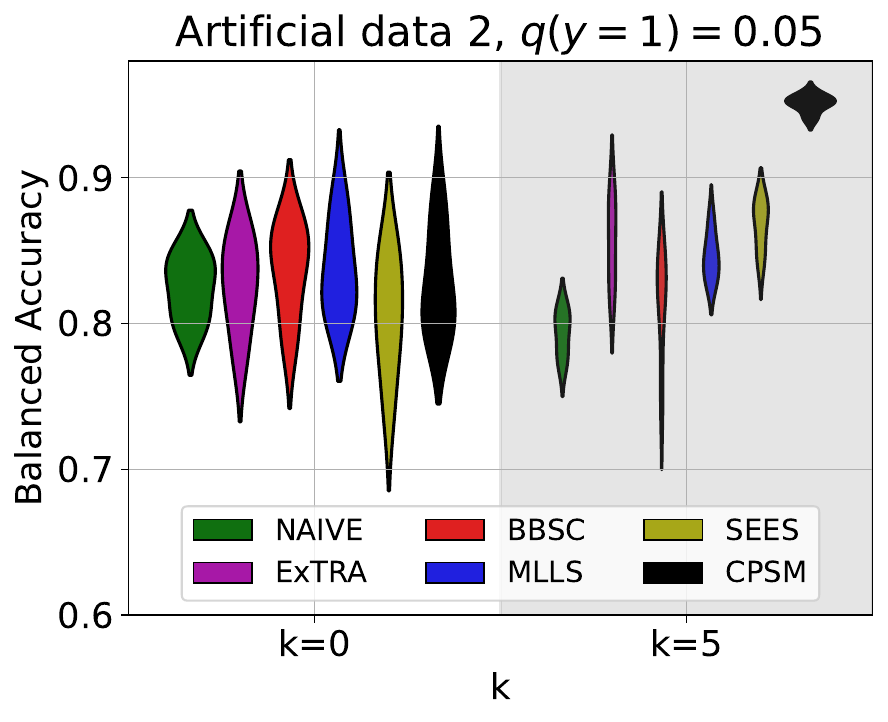} &
      \includegraphics[width=0.4\textwidth]{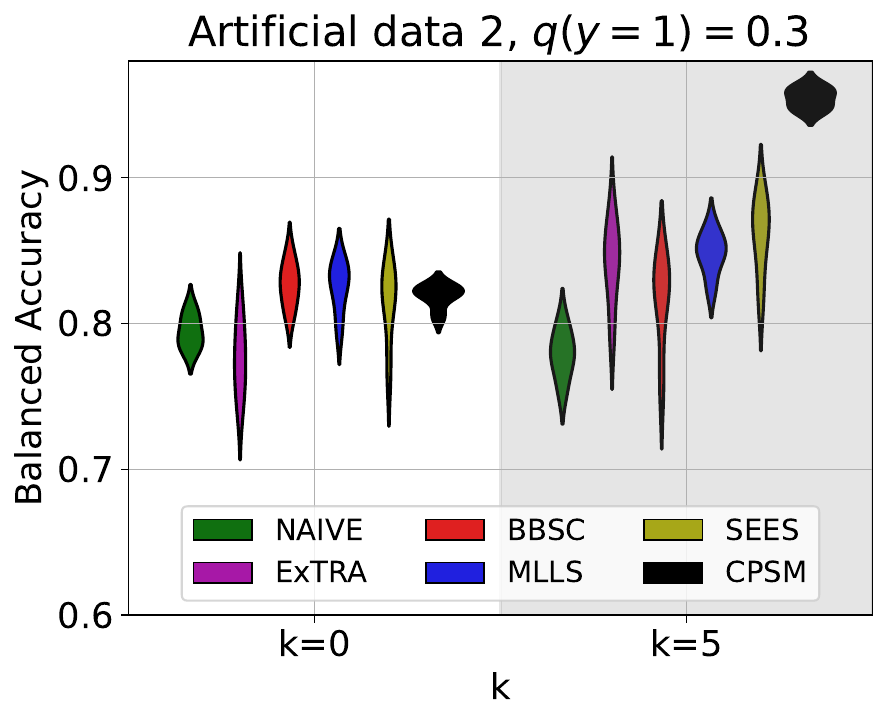}  \\
      \includegraphics[width=0.4\textwidth]{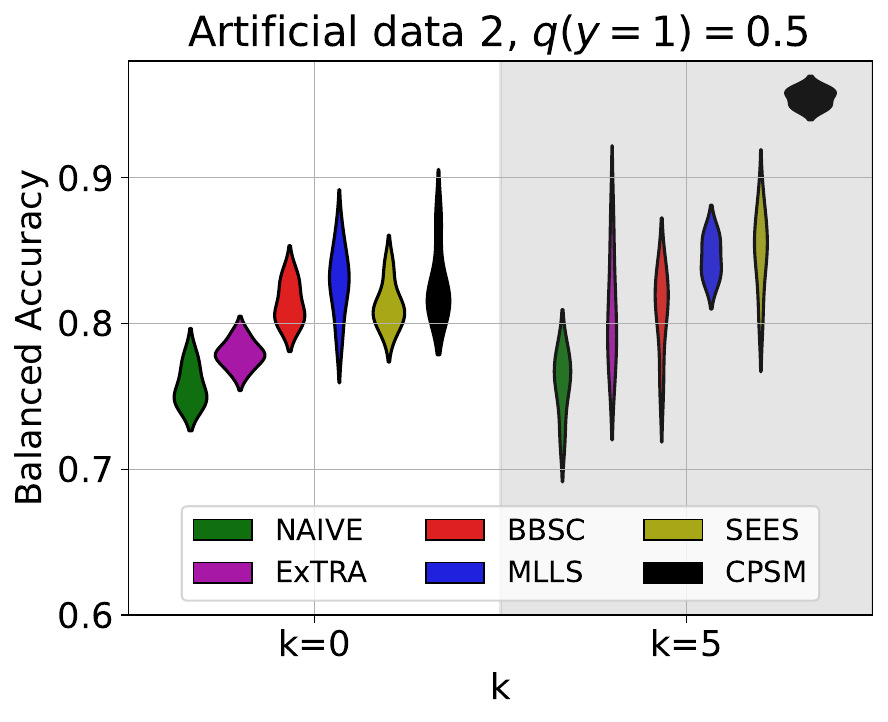} &
      \includegraphics[width=0.4\textwidth]{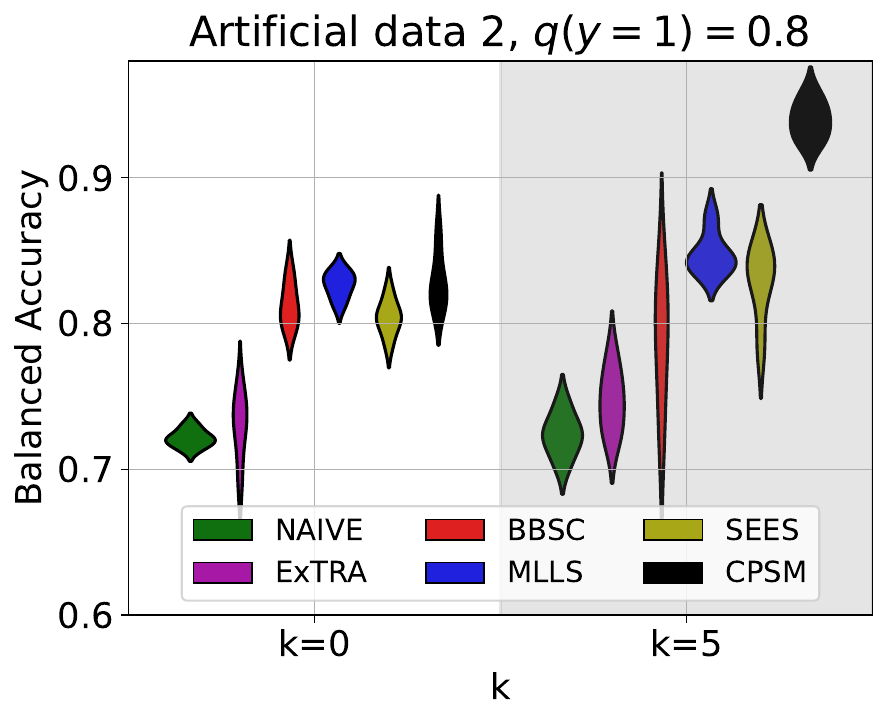}  \\
      \end{tabular}
    \caption{Balanced Accuracy for artificial data 2, for $p(y=1)=0.05$ and varying class prior for target data $q(y=1)=0.05,0.3,0.5,0.8$.}
    \label{fig:artificial1a}
\end{figure}

\begin{figure}
\centering
    \begin{tabular}{c c}
      \includegraphics[width=0.4\textwidth]{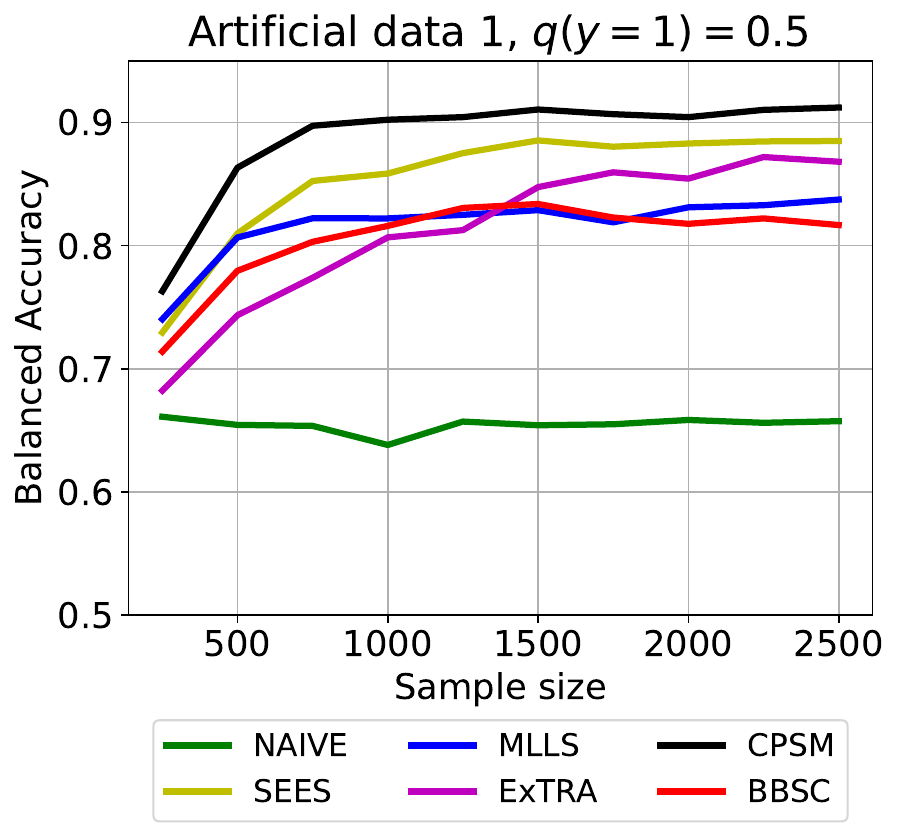} &
      \includegraphics[width=0.4\textwidth]{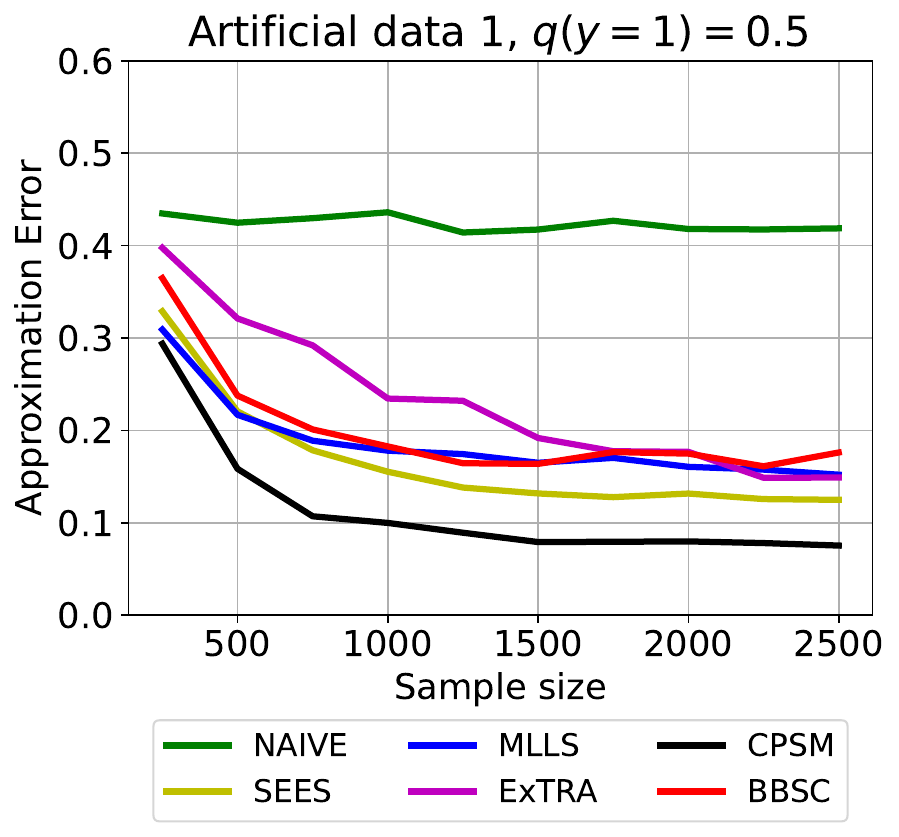}  \\
      \includegraphics[width=0.4\textwidth]{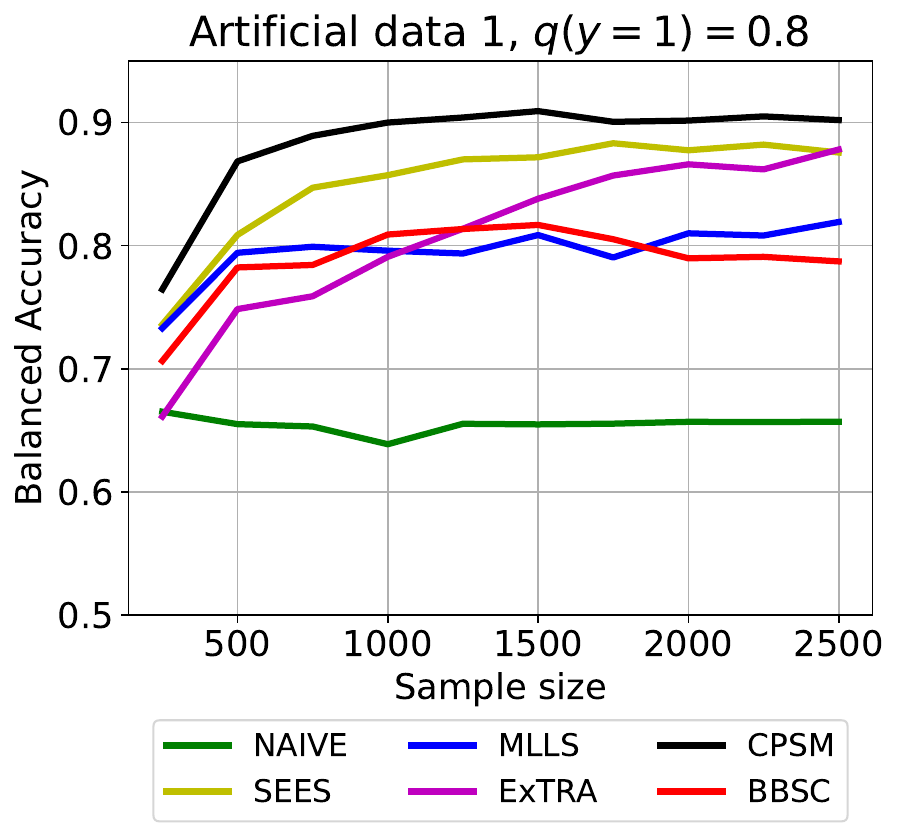} &
       \includegraphics[width=0.4\textwidth]{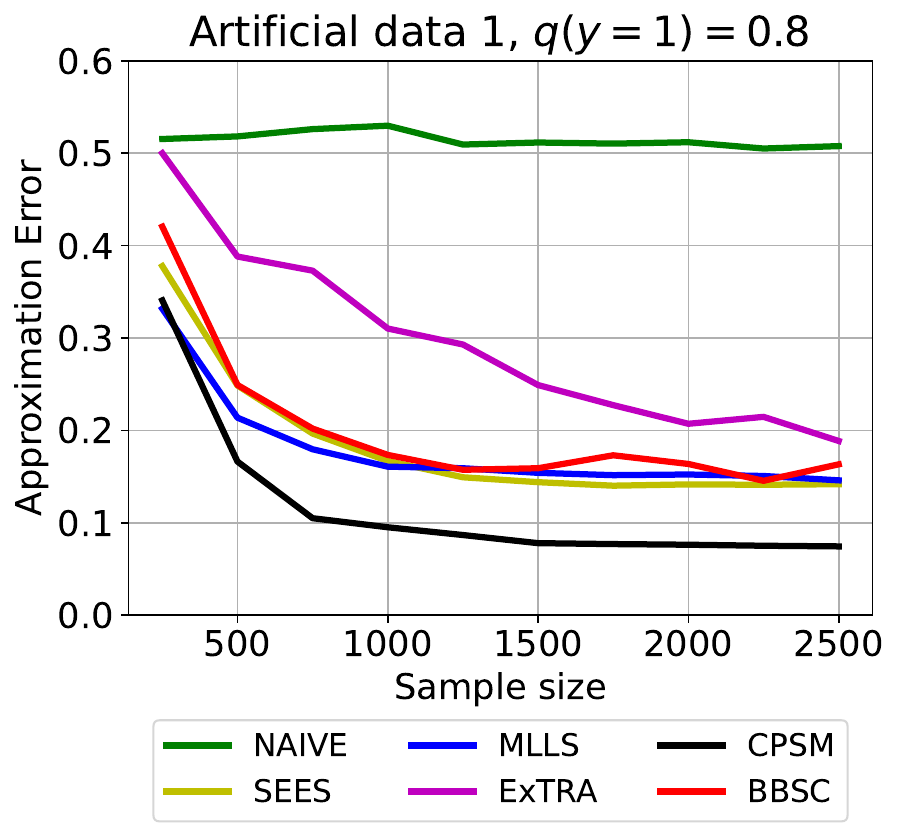}  \\
      \end{tabular}
    \caption{Balanced Accuracy and Approximation Error wrt sample size for artificial data 1,  $p(y=1)=0.05$ and $k=5$. The class prior for the target data is $q(y=1)=0.5$ (top panels) and $q(y=1)=0.8$ (bottom panels).}
    \label{fig:artificial2}
\end{figure}

\begin{figure}
\centering
    \begin{tabular}{c c}
      \includegraphics[width=0.4\textwidth]{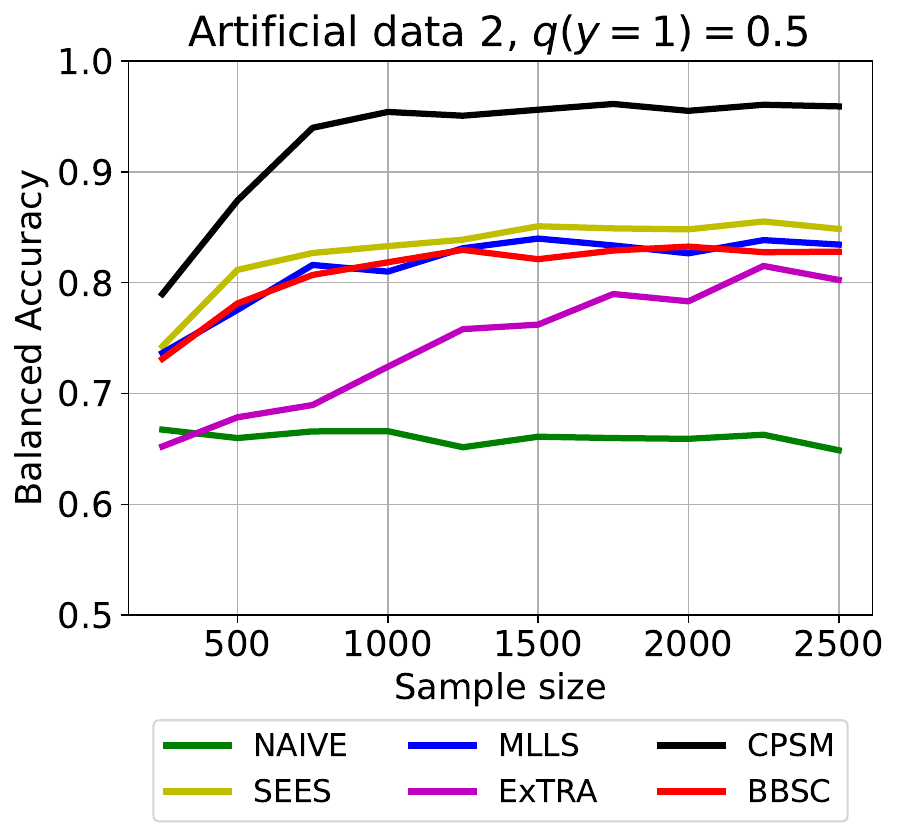} &
      \includegraphics[width=0.4\textwidth]{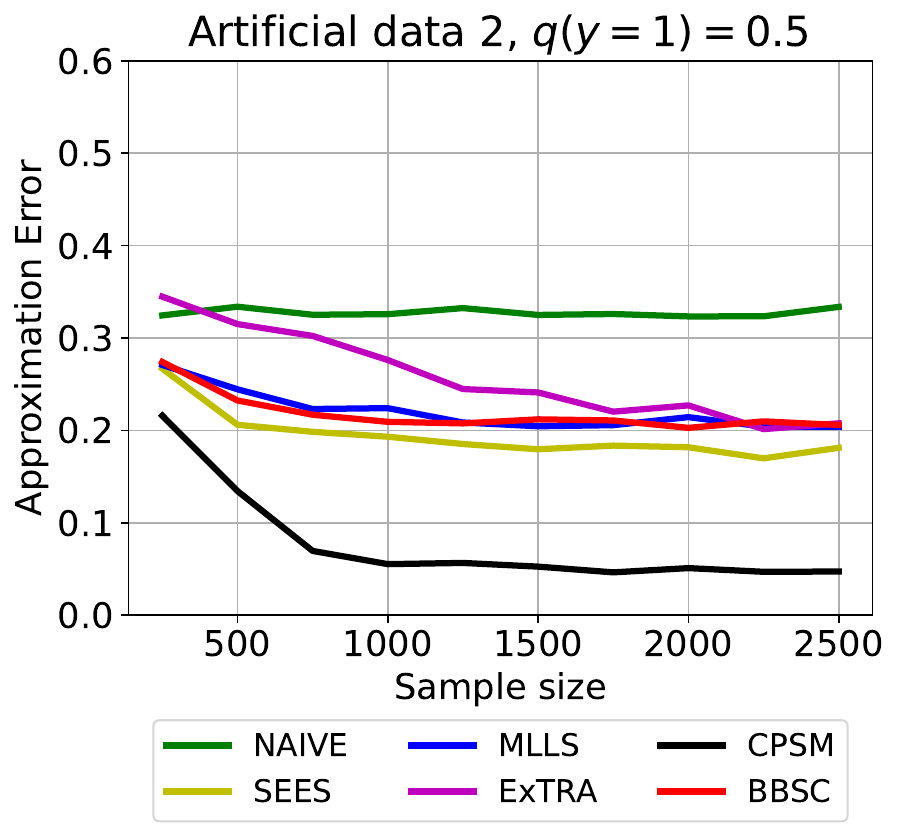}  \\
      \includegraphics[width=0.4\textwidth]{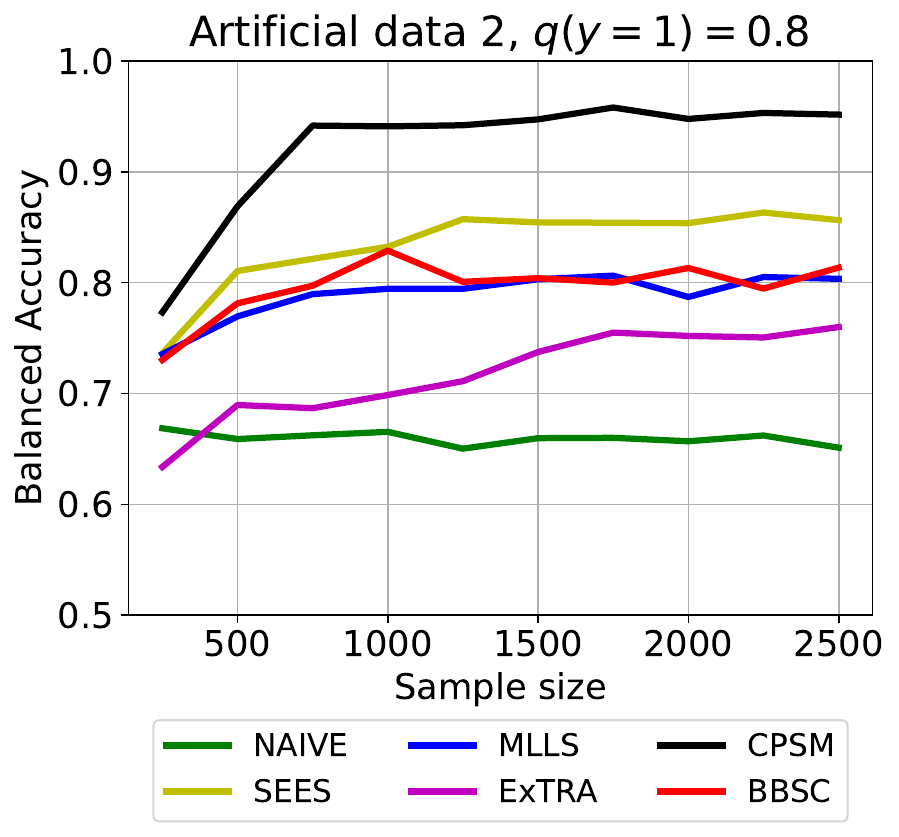} &
       \includegraphics[width=0.4\textwidth]{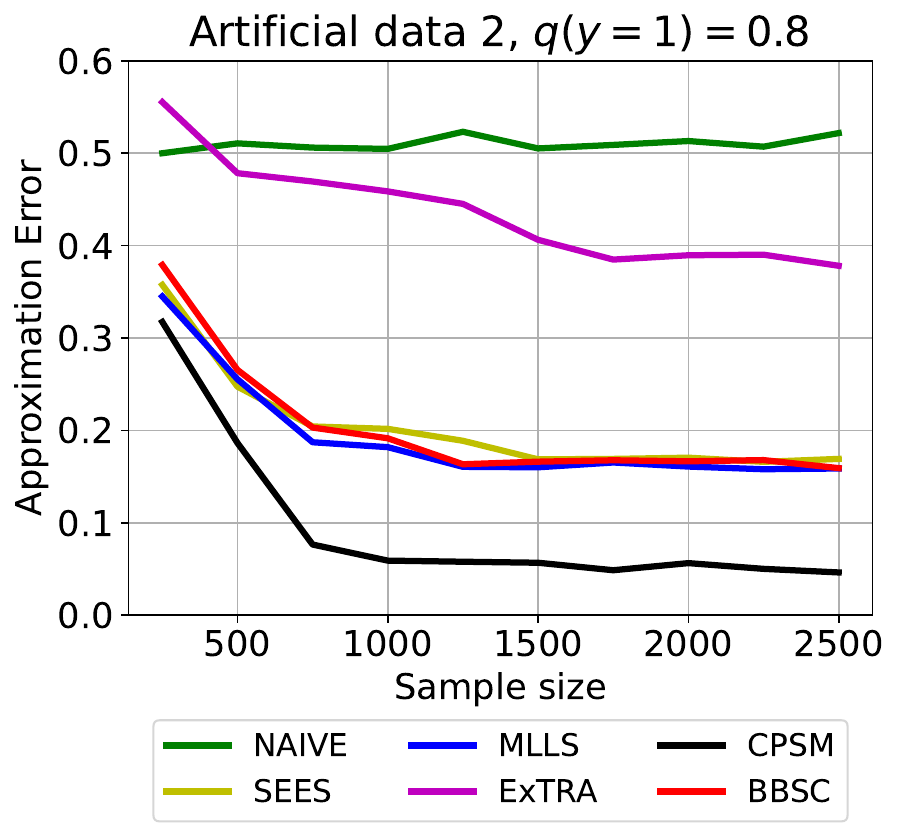}  \\
      \end{tabular}
    \caption{Balanced Accuracy and Approximation Error wrt sample size for artificial data 2,  $p(y=1)=0.05$ and $k=5$. The class prior for the target data is $q(y=1)=0.5$ (top panels) and $q(y=1)=0.8$ (bottom panels).}
    \label{fig:artificial2a}
\end{figure}

\subsection{Case study: MIMIC database}
We perform experiments on large clinical database MIMIC \cite{MIMICIII}, containing information on $33166$ patients of
various intensive care units whose  diagnosis is recorded using the coding
scheme ICD-9. In the analyses we focused on adult patients  only  ($>$16 years).
As class variables we consider indicators of 4 families of diseases, which were already used in previous studies: COPD (Chronic Obstructive Pulmonary Disease; ICD-9 codes 490–496), DIABETES (Diabetes mellitus, ICD-9 codes: 249–250), KIDNEY (Kidney disease; ICD-9 codes: 580–589) and FLUID (Fluid electrolyte disease, ICD-9 codes: 276) \cite{Bromurietal2014}.
The percentage of patients with individual diseases was $7\%$, $14\%$, $38\%$ and $37\%$ for COPD, DIABETES, KIDNEY and FLUID, respectively.
The considered dataset contains $308$ numerical features, most of which correspond to some laboratory
readings. The largest group of features are blood and diagnostic tests (e.g. Glucose, Sodium readings, etc.). In addition, there are medical scores used to track a person’s status
during the stay in an intensive care unit, such as Braden score used to assess a risk of developing a pressure ulcer. Finally, some administrative information such as age, gender, and marital status is also available. 
Pre-processing (normalization of variables, removal of missing values) was performed following the steps used by other authors using this dataset; details can be found in \cite{Bromurietal2014, ZUFFEREY2015, TEISSEYRE2019}.

For the MIMIC dataset  we introduce conditional probability shift as follows. 
The conditioning variable $z$  is either age or gender.
The patient's age is discretized into two categories '$<$60 years' ($z=0$) and '$>$60 years' ($z=1$). 
In case of the gender variable, the value $z=1$ refers to  males, and the value $z=0$ refers to  females.
Age and gender are natural choices for a conditioning variable because most of the other variables are clinical symptoms and diagnostic test results.  
We artificially sample source data  in such a way that $p(y=1|z=0)=p(y=1|z=1)=p(y=1)=a$, where $a$ is  a varying  parameter. Moreover, for  the target data, in order to allow for shifts, we sample observations  in such a way  that $q(y=1|z=0)=a$, and $q(y=1|z=1)=a+k$, where $k$ is an additional parameter that is  controlled. By choosing larger  $k$, we increase the dependence between $y$ and $z$ in the target data. 
The most interesting scenario is when the prevalence of the disease in the source data is low. In this way, the source data can resemble data collected before the increase in disease incidence, e.g. before the outbreak of the pandemic. Therefore, in the experiments, we analyze the results for $a=0.05,0.1, 0.2$. 
Moreover, we show the results for $k=0.3,0.5,0.7$. Note that $k>0$ corresponds to a shift of conditional distribution of $y$ given $z$.

In the main part of the paper we show and discuss in detail the results for the situation when the basic model is logistic regression and the conditioning variable is age. The remaining results can be found in the appendix.
Tables \ref{tab:acc_logit} and \ref{tab:error_logit} contain the values of Balanced Accuracy and Approximation Error for the considered method, when logistic regression (LOGIT) is used as a base classifier. 
The proposed CPSM emerges as the winner in terms of Balanced Accuracy and Approximation Error for most settings of parameters $a$ and $k$, while SEES usually comes second. The advantage of CPSM over SEES is observed for $a=0.05$ and $a=0.1$, while for $a=2$, SEES usually performs better. As expected, in the case of CPSM and SEES, the classification accuracy increases with increasing $k$, which is related to the fact that the discrepancy between the conditional distributions $p(y=1|z)$ and $q(y=1|z)$ increases.
We also observe a similar effect for other methods, such as BBSC and MLLS.
Tables \ref{tab:acc_logit_sex} and \ref{tab:error_logit_sex}  (Appendix) show the results for the situation when the conditioning variable is gender. The conclusions are very similar to those for the age variable.

Tables \ref{tab:acc_dnn} and \ref{tab:error_dnn} show analogous results when the neural network (NN) is used as the base classifier.
The performance when  using NN as a base classifier compared to the logistic model depends on the considered class variable and the parameter setting. For example, in the case of COPD, for $a=0.1$ and $k=0.7$, the Balanced Accuracy is 0.842 for CPSM + LOGIT and 0.752 for CPSM + NN. On the other hand, in the case of COPD, for $a=0.2$ and $k=0.3$, the Balanced Accuracy is 0.65 for CPSM + LOGIT and 0.799 for CPSM + NN.
In general, when using NN as a base classifier, SEES usually performs slightly better than CPSM in terms of Balanced Accuracy, but CPSM still has the smallest approximation  errors in most cases.

We also analyzed the computation times for individual methods. Table \ref{tab:comp_times} shows the times averaged over 5 runs for the MIMIC dataset. 
As expected, the proposed CPSM method is significantly  slower than the MLLS method, which is also based on the EM scheme. This is due to the fact that in CPSM, the M-step requires the estimation of the multinomial model parameters, while in MLLS, the M-step is straightforward and only inlolves computing prior probability by averaging the posterior probabilities. Importantly, however, MLLS and BBSC are based on the more restrictive LS scheme. Among the methods based on the more general SJS scheme (CPSM and SEES), the CPSM method is clearly faster than the SEES method. The computation times depend, of course, on the choice of the base classifier, in the case of using NN they are significantly higher than in the case of the logistic model.

\begin{table}[ht!]
\caption{{\bf Computation times} [sec] for the considered methods in the case of  MIMIC data and two base classifiers: logistic regression (LOGIT) and neural network (NN). The results are averaged over 5 runs.}
\label{tab:comp_times}
\begin{tabular}{l|llllll}
\toprule 
Classifier & NAIVE & ExTRA \cite{MaityYurochkinBanerjeeSun2023} & BBSC \cite{LiptonWangSmola2018} & MLLS  \cite{SaerensLatinneDecaestecker2002} & SEES \cite{ChenZahariaZou2024} & CPSM \\
\midrule 
 LOGIT & 0.122 &	28.501  &	0.142  &	0.123  &	 428.818  &	 91.416  \\
NN &158.435  &	182.979  &	235.506  &	158.56  &	 645.358  &	 248.322  \\
\bottomrule
\end{tabular}
\end{table}

\section{Conclusions}
In this work, we consider the  practically important  case of distribution shift between the source and target datasets, such that the conditional distribution of the class variable, given selected features, changes. The proposed CPSM method is based on modeling the conditional distributions using multinomial regression. We show that estimation of the multinomial regression parameters for the target data is feasible  using the EM algorithm. This leads to an algorithm for estimating the posterior distribution for the target data.

Experiments, conducted on artificial and medical data, show that, in the considered cases, CPSM  is superior with  regard  to balanced classification accuracy and  approximation errors to methods based on the LS and SJS assumptions (BBSC, MLLS and SEES). In particular, CPSM performs significantly better in the case when the source and target conditional distributions differ, while the priors remain unchanged. The advantage of CPSM over competing methods is significant regardless of the sample size and the prior probability of the class variable in the source data. Finally, CPSM is also computationally faster than the competing method SEES, which is the only one based, apart from CPSM, on the SJS assumption.

Several  problems of interest  for future research can be identified. It would be valuable to test whether the shift in the conditional distributions is statistically significant, e.g. by testing whether the parameters of the multinomial model for the source and target data differ. Moreover, as in the present work we have shown the advantage of  (\ref{CPS}) for inference, it is of interest to identify a  set of conditioning variables $z$ for which it holds, based on data. Also, checking whether  considered procedures are robust to small violations
of (\ref{CPS}) is desirable.



\begin{table}[ht!]
\caption{{\bf Balanced Accuracy score} for MIMIC case study, for the {\bf conditioning variable: age}. Four diseases are considered: COPD, DIABETES, KIDNEY and FLUID. The method with the highest Balanced Accuracy is in red and the second best in blue. {\bf Logistic regression} is used as a base classifier.}
\label{tab:acc_logit}
\begin{tabular}{l|l|llllll}
\toprule 
\multicolumn{8}{c}{Disease: COPD} \\
\midrule
a & k & NAIVE & ExTRA \cite{MaityYurochkinBanerjeeSun2023} & BBSC \cite{LiptonWangSmola2018} & MLLS  \cite{SaerensLatinneDecaestecker2002} & SEES \cite{ChenZahariaZou2024} & CPSM \\
\midrule 
0.05 &	0.3 &	0.51 $\pm$ 0.005 &	0.552 $\pm$ 0.052 &	\cellcolor{blue!25} 0.578 $\pm$ 0.013 &	0.567 $\pm$ 0.01 &	0.517 $\pm$ 0.008 &	\cellcolor{red!25} 0.608 $\pm$ 0.015 \\
     &	0.5 &	0.509 $\pm$ 0.003 &	0.551 $\pm$ 0.059 &	0.592 $\pm$ 0.018 &	\cellcolor{blue!25} 0.612 $\pm$ 0.014 &	0.563 $\pm$ 0.029 &	\cellcolor{red!25} 0.675 $\pm$ 0.011 \\
     &	0.7 &	0.508 $\pm$ 0.004 &	0.587 $\pm$ 0.036 &	0.612 $\pm$ 0.032 &	\cellcolor{blue!25} 0.626 $\pm$ 0.011 &	0.608 $\pm$ 0.031 &	\cellcolor{red!25} 0.796 $\pm$ 0.016 \\
\midrule
0.1 &	0.3 &	0.527 $\pm$ 0.01 &	0.586 $\pm$ 0.045 &	0.586 $\pm$ 0.013 &	\cellcolor{blue!25} 0.601 $\pm$ 0.01 &	0.575 $\pm$ 0.014 &	\cellcolor{red!25} 0.641 $\pm$ 0.012 \\
    &	0.5 &	0.528 $\pm$ 0.006 &	0.568 $\pm$ 0.048 &	0.619 $\pm$ 0.016 &	0.622 $\pm$ 0.014 &	\cellcolor{blue!25} 0.697 $\pm$ 0.007 &	\cellcolor{red!25} 0.72 $\pm$ 0.008 \\
    &	0.7 &	0.524 $\pm$ 0.008 &	0.584 $\pm$ 0.026 &	0.623 $\pm$ 0.023 &	0.635 $\pm$ 0.005 &	\cellcolor{blue!25} 0.774 $\pm$ 0.029 &	\cellcolor{red!25} 0.842 $\pm$ 0.006 \\
\midrule
0.2 &	0.3 &	0.575 $\pm$ 0.01 &	0.542 $\pm$ 0.047 &	0.608 $\pm$ 0.009 &	0.605 $\pm$ 0.009 &	\cellcolor{blue!25} 0.631 $\pm$ 0.008 &	\cellcolor{red!25} 0.65 $\pm$ 0.01 \\
    &	0.5 &	0.573 $\pm$ 0.014 &	0.622 $\pm$ 0.025 &	0.632 $\pm$ 0.012 &	0.627 $\pm$ 0.01 &	\cellcolor{red!25} 0.784 $\pm$ 0.002 &	\cellcolor{blue!25} 0.743 $\pm$ 0.01 \\
    &	0.7 &	0.582 $\pm$ 0.012 &	0.546 $\pm$ 0.054 &	0.634 $\pm$ 0.01 &	0.641 $\pm$ 0.005 &	\cellcolor{red!25} 0.855 $\pm$ 0.002 &	\cellcolor{blue!25} 0.845 $\pm$ 0.009 \\
\bottomrule
\end{tabular}
\begin{tabular}{l|l|llllll}
\multicolumn{8}{c}{Disease: DIABETES} \\
\midrule
a & k & NAIVE & ExTRA \cite{MaityYurochkinBanerjeeSun2023} & BBSC \cite{LiptonWangSmola2018} & MLLS  \cite{SaerensLatinneDecaestecker2002} & SEES \cite{ChenZahariaZou2024} & CPSM \\
\midrule
0.05 &	0.3 &	0.548 $\pm$ 0.012 &	0.561 $\pm$ 0.023 &	\cellcolor{blue!25} 0.625 $\pm$ 0.019 &	0.616 $\pm$ 0.017 &	0.579 $\pm$ 0.021 &	\cellcolor{red!25} 0.637 $\pm$ 0.007 \\
     &	0.5 &	0.549 $\pm$ 0.009 &	0.557 $\pm$ 0.019 &	0.656 $\pm$ 0.015 &	0.641 $\pm$ 0.014 &	\cellcolor{blue!25} 0.668 $\pm$ 0.034 &	\cellcolor{red!25} 0.73 $\pm$ 0.025 \\
     &	0.7 &	0.548 $\pm$ 0.01 &	0.544 $\pm$ 0.023 &	0.663 $\pm$ 0.02 &	0.661 $\pm$ 0.016 &	\cellcolor{blue!25} 0.716 $\pm$ 0.037 &	\cellcolor{red!25} 0.828 $\pm$ 0.03 \\
\midrule
0.1 &	0.3 &	0.577 $\pm$ 0.01 &	0.593 $\pm$ 0.014 &	0.639 $\pm$ 0.015 &	0.638 $\pm$ 0.011 &	\cellcolor{blue!25} 0.644 $\pm$ 0.01 &	\cellcolor{red!25} 0.669 $\pm$ 0.012 \\
    &	0.5 &	0.579 $\pm$ 0.011 &	0.58 $\pm$ 0.016 &	0.686 $\pm$ 0.015 &	0.665 $\pm$ 0.015 &	\cellcolor{blue!25} 0.772 $\pm$ 0.014 &	\cellcolor{red!25} 0.776 $\pm$ 0.013 \\
    &	0.7 &	0.577 $\pm$ 0.007 &	0.567 $\pm$ 0.036 &	0.697 $\pm$ 0.019 &	0.688 $\pm$ 0.028 &	\cellcolor{blue!25} 0.852 $\pm$ 0.01 &	\cellcolor{red!25} 0.861 $\pm$ 0.004 \\
\midrule
0.2 &	0.3 &	0.629 $\pm$ 0.01 &	0.672 $\pm$ 0.026 &	0.663 $\pm$ 0.008 &	0.648 $\pm$ 0.01 &	\cellcolor{blue!25} 0.698 $\pm$ 0.017 &	\cellcolor{red!25} 0.708 $\pm$ 0.007 \\
    &	0.5 &	0.63 $\pm$ 0.011 &	0.653 $\pm$ 0.029 &	0.683 $\pm$ 0.012 &	0.691 $\pm$ 0.019 &	\cellcolor{red!25} 0.811 $\pm$ 0.013 &	\cellcolor{blue!25} 0.797 $\pm$ 0.013 \\
    &	0.7 &	0.631 $\pm$ 0.009 &	0.701 $\pm$ 0.04 &	0.727 $\pm$ 0.028 &	0.707 $\pm$ 0.012 &	\cellcolor{red!25} 0.874 $\pm$ 0.006 &	\cellcolor{blue!25} 0.87 $\pm$ 0.004 \\
\bottomrule
\end{tabular}
\begin{tabular}{l|l|llllll}
\multicolumn{8}{c}{Disease: KIDNEY} \\
\midrule
a & k & NAIVE & ExTRA \cite{MaityYurochkinBanerjeeSun2023} & BBSC \cite{LiptonWangSmola2018} & MLLS  \cite{SaerensLatinneDecaestecker2002} & SEES \cite{ChenZahariaZou2024} & CPSM \\
\midrule
0.05 &	0.3 &	0.601 $\pm$ 0.012 &	0.644 $\pm$ 0.029 &	\cellcolor{blue!25} 0.734 $\pm$ 0.011 &	0.733 $\pm$ 0.015 &	0.69 $\pm$ 0.013 &	\cellcolor{red!25} 0.771 $\pm$ 0.022 \\
    &	0.5 &	0.592 $\pm$ 0.015 &	0.668 $\pm$ 0.02 &	0.76 $\pm$ 0.014 &	\cellcolor{blue!25} 0.766 $\pm$ 0.013 &	0.761 $\pm$ 0.024 &	\cellcolor{red!25} 0.825 $\pm$ 0.014 \\
    &	0.7 &	0.602 $\pm$ 0.012 &	0.663 $\pm$ 0.014 &	0.774 $\pm$ 0.014 &	0.788 $\pm$ 0.011 &	\cellcolor{blue!25} 0.807 $\pm$ 0.022 &	\cellcolor{red!25} 0.885 $\pm$ 0.003 \\
\midrule
0.1 &	0.3 &	0.663 $\pm$ 0.01 &	0.698 $\pm$ 0.021 &	0.757 $\pm$ 0.006 &	0.76 $\pm$ 0.01 &	\cellcolor{blue!25} 0.761 $\pm$ 0.007 &	\cellcolor{red!25} 0.79 $\pm$ 0.02 \\
    &	0.5 &	0.652 $\pm$ 0.009 &	0.699 $\pm$ 0.081 &	0.781 $\pm$ 0.01 &	0.782 $\pm$ 0.012 &	\cellcolor{blue!25} 0.837 $\pm$ 0.018 &	\cellcolor{red!25} 0.845 $\pm$ 0.008 \\
    &	0.7 &	0.661 $\pm$ 0.009 &	0.711 $\pm$ 0.056 &	0.802 $\pm$ 0.013 &	0.806 $\pm$ 0.008 &	\cellcolor{blue!25} 0.89 $\pm$ 0.013 &	\cellcolor{red!25} 0.896 $\pm$ 0.003 \\
\midrule
0.2 &	0.3 &	0.736 $\pm$ 0.006 &	0.766 $\pm$ 0.02 &	0.767 $\pm$ 0.006 &	0.771 $\pm$ 0.01 &	\cellcolor{blue!25} 0.795 $\pm$ 0.007 &	\cellcolor{red!25} 0.803 $\pm$ 0.008 \\
    &	0.5 &	0.735 $\pm$ 0.003 &	0.761 $\pm$ 0.047 &	0.8 $\pm$ 0.006 &	0.792 $\pm$ 0.002 &	\cellcolor{red!25} 0.863 $\pm$ 0.007 &	\cellcolor{blue!25} 0.857 $\pm$ 0.005 \\
    &	0.7 &	0.738 $\pm$ 0.005 &	0.769 $\pm$ 0.018 &	0.818 $\pm$ 0.008 &	0.808 $\pm$ 0.006 &	\cellcolor{red!25} 0.901 $\pm$ 0.005 &	\cellcolor{blue!25} 0.899 $\pm$ 0.004 \\
\bottomrule
\end{tabular}
\begin{tabular}{l|l|llllll}
\multicolumn{8}{c}{Disease: FLUID} \\
\midrule
a & k & NAIVE & ExTRA \cite{MaityYurochkinBanerjeeSun2023} & BBSC \cite{LiptonWangSmola2018} & MLLS  \cite{SaerensLatinneDecaestecker2002} & SEES \cite{ChenZahariaZou2024} & CPSM \\
\midrule
0.05 &	0.3 &	0.511 $\pm$ 0.002 &	0.548 $\pm$ 0.025 &	0.555 $\pm$ 0.015 &	\cellcolor{blue!25} 0.562 $\pm$ 0.017 &	0.515 $\pm$ 0.006 &	\cellcolor{red!25} 0.584 $\pm$ 0.01 \\
     &	0.5 &	0.508 $\pm$ 0.003 &	0.562 $\pm$ 0.034 &	\cellcolor{blue!25} 0.587 $\pm$ 0.028 &	0.577 $\pm$ 0.011 &	0.536 $\pm$ 0.019 &	\cellcolor{red!25} 0.663 $\pm$ 0.046 \\
    &	0.7 &	0.509 $\pm$ 0.003 &	0.563 $\pm$ 0.019 &	0.595 $\pm$ 0.026 &	\cellcolor{blue!25} 0.619 $\pm$ 0.02 &	0.543 $\pm$ 0.028 &	\cellcolor{red!25} 0.759 $\pm$ 0.028 \\
\midrule
0.1 &	0.3 &	0.526 $\pm$ 0.006 &	0.554 $\pm$ 0.034 &	\cellcolor{blue!25} 0.595 $\pm$ 0.024 &	0.594 $\pm$ 0.011 &	0.576 $\pm$ 0.021 &	\cellcolor{red!25} 0.619 $\pm$ 0.008 \\
 &	0.5 &	0.518 $\pm$ 0.003 &	0.538 $\pm$ 0.025 &	0.618 $\pm$ 0.019 &	0.618 $\pm$ 0.011 &	\cellcolor{blue!25} 0.642 $\pm$ 0.02 &	\cellcolor{red!25} 0.736 $\pm$ 0.027 \\
 &	0.7 &	0.52 $\pm$ 0.006 &	0.56 $\pm$ 0.023 &	0.645 $\pm$ 0.012 &	0.652 $\pm$ 0.026 &	\cellcolor{blue!25} 0.69 $\pm$ 0.047 &	\cellcolor{red!25} 0.83 $\pm$ 0.021 \\
\midrule
0.2 &	0.3 &	0.567 $\pm$ 0.012 &	0.573 $\pm$ 0.044 &	0.609 $\pm$ 0.013 &	0.605 $\pm$ 0.008 &	\cellcolor{blue!25} 0.634 $\pm$ 0.016 &	\cellcolor{red!25} 0.66 $\pm$ 0.01 \\
 &	0.5 &	0.569 $\pm$ 0.01 &	0.56 $\pm$ 0.098 &	0.643 $\pm$ 0.014 &	0.642 $\pm$ 0.015 &	\cellcolor{red!25} 0.781 $\pm$ 0.007 &	\cellcolor{blue!25} 0.772 $\pm$ 0.004 \\
 &	0.7 &	0.569 $\pm$ 0.009 &	0.612 $\pm$ 0.031 &	0.666 $\pm$ 0.016 &	0.661 $\pm$ 0.008 &	\cellcolor{red!25} 0.862 $\pm$ 0.003 &	\cellcolor{blue!25} 0.85 $\pm$ 0.002 \\
\bottomrule
\end{tabular}
\end{table}


\begin{table}[ht!]
\caption{{\bf Approximation Error} for MIMIC case study, for the {\bf conditioning variable: age}. Four diseases are considered: COPD, DIABETES, KIDNEY and FLUID. The method with the smallest approximation  error is in red and the second best in blue. {\bf Logistic regression} is used as a base classifier.}
\label{tab:error_logit}
\begin{tabular}{l|l|llllll}
\toprule 
\multicolumn{8}{c}{Disease: COPD} \\
\midrule
a & k & NAIVE & ExTRA \cite{MaityYurochkinBanerjeeSun2023} & BBSC \cite{LiptonWangSmola2018} & MLLS  \cite{SaerensLatinneDecaestecker2002} & SEES \cite{ChenZahariaZou2024} & CPSM \\
\midrule 
0.05 &	0.3 &	0.203 $\pm$ 0.002 &	0.172 $\pm$ 0.016 &	0.155 $\pm$ 0.021 &	\cellcolor{blue!25} 0.144 $\pm$ 0.006 &	0.19 $\pm$ 0.01 &	\cellcolor{red!25} 0.093 $\pm$ 0.009 \\
 &	0.5 &	0.259 $\pm$ 0.002 &	0.227 $\pm$ 0.012 &	0.231 $\pm$ 0.039 &	0.208 $\pm$ 0.002 &	\cellcolor{blue!25} 0.203 $\pm$ 0.019 &	\cellcolor{red!25} 0.101 $\pm$ 0.013 \\
 &	0.7 &	0.299 $\pm$ 0.005 &	0.256 $\pm$ 0.011 &	0.505 $\pm$ 0.161 &	0.27 $\pm$ 0.01 &	\cellcolor{blue!25} 0.221 $\pm$ 0.023 &	\cellcolor{red!25} 0.106 $\pm$ 0.011 \\
\midrule
0.1 &	0.3 &	0.16 $\pm$ 0.002 &	0.164 $\pm$ 0.036 &	0.136 $\pm$ 0.005 &	0.132 $\pm$ 0.002 &	\cellcolor{blue!25} 0.107 $\pm$ 0.013 &	\cellcolor{red!25} 0.062 $\pm$ 0.005 \\
 &	0.5 &	0.218 $\pm$ 0.003 &	0.217 $\pm$ 0.04 &	0.212 $\pm$ 0.008 &	0.208 $\pm$ 0.003 &	\cellcolor{blue!25} 0.087 $\pm$ 0.005 &	\cellcolor{red!25} 0.069 $\pm$ 0.004 \\
 &	0.7 &	0.258 $\pm$ 0.006 &	0.28 $\pm$ 0.053 &	0.275 $\pm$ 0.013 &	0.278 $\pm$ 0.009 &	\cellcolor{blue!25} 0.113 $\pm$ 0.017 &	\cellcolor{red!25} 0.078 $\pm$ 0.009 \\
\midrule
0.2 &	0.3 &	0.137 $\pm$ 0.002 &	0.181 $\pm$ 0.051 &	0.128 $\pm$ 0.005 &	0.13 $\pm$ 0.001 &	\cellcolor{red!25} 0.041 $\pm$ 0.008 &	\cellcolor{blue!25} 0.055 $\pm$ 0.007 \\
 &	0.5 &	0.205 $\pm$ 0.005 &	0.208 $\pm$ 0.01 &	0.214 $\pm$ 0.004 &	0.213 $\pm$ 0.004 &	\cellcolor{red!25} 0.057 $\pm$ 0.011 &	\cellcolor{blue!25} 0.06 $\pm$ 0.01 \\
 &	0.7 &	0.254 $\pm$ 0.003 &	0.292 $\pm$ 0.043 &	0.279 $\pm$ 0.005 &	0.283 $\pm$ 0.005 &	\cellcolor{red!25} 0.033 $\pm$ 0.008 &	\cellcolor{blue!25} 0.067 $\pm$ 0.006 \\
\bottomrule
\end{tabular}
\begin{tabular}{l|l|llllll}
\multicolumn{8}{c}{Disease: DIABETES} \\
\midrule
a & k & NAIVE & ExTRA \cite{MaityYurochkinBanerjeeSun2023} & BBSC \cite{LiptonWangSmola2018} & MLLS  \cite{SaerensLatinneDecaestecker2002} & SEES \cite{ChenZahariaZou2024} & CPSM \\
\midrule
0.05 &	0.3 &	0.219 $\pm$ 0.008 &	0.198 $\pm$ 0.009 &	\cellcolor{blue!25} 0.135 $\pm$ 0.014 &	0.135 $\pm$ 0.009 &	0.16 $\pm$ 0.014 &	\cellcolor{red!25} 0.09 $\pm$ 0.014 \\
 &	0.5 &	0.29 $\pm$ 0.011 &	0.257 $\pm$ 0.023 &	0.198 $\pm$ 0.016 &	0.209 $\pm$ 0.009 &	\cellcolor{blue!25} 0.158 $\pm$ 0.024 &	\cellcolor{red!25} 0.101 $\pm$ 0.022 \\
 &	0.7 &	0.343 $\pm$ 0.014 &	0.329 $\pm$ 0.019 &	0.248 $\pm$ 0.009 &	0.266 $\pm$ 0.013 &	\cellcolor{blue!25} 0.178 $\pm$ 0.03 &	\cellcolor{red!25} 0.098 $\pm$ 0.031 \\
\midrule
0.1 &	0.3 &	0.174 $\pm$ 0.007 &	0.137 $\pm$ 0.013 &	0.122 $\pm$ 0.007 &	0.12 $\pm$ 0.003 &	\cellcolor{blue!25} 0.082 $\pm$ 0.01 &	\cellcolor{red!25} 0.064 $\pm$ 0.005 \\
 &	0.5 &	0.245 $\pm$ 0.009 &	0.238 $\pm$ 0.02 &	0.18 $\pm$ 0.004 &	0.194 $\pm$ 0.007 &	\cellcolor{blue!25} 0.069 $\pm$ 0.011 &	\cellcolor{red!25} 0.068 $\pm$ 0.008 \\
 &	0.7 &	0.295 $\pm$ 0.013 &	0.289 $\pm$ 0.024 &	0.241 $\pm$ 0.014 &	0.259 $\pm$ 0.011 &	\cellcolor{blue!25} 0.072 $\pm$ 0.011 &	\cellcolor{red!25} 0.063 $\pm$ 0.018 \\
\midrule
0.2 &	0.3 &	0.124 $\pm$ 0.004 &	0.106 $\pm$ 0.023 &	0.107 $\pm$ 0.002 &	0.109 $\pm$ 0.004 &	\cellcolor{red!25} 0.038 $\pm$ 0.012 &	\cellcolor{blue!25} 0.041 $\pm$ 0.007 \\
 &	0.5 &	0.201 $\pm$ 0.007 &	0.191 $\pm$ 0.034 &	0.181 $\pm$ 0.008 &	0.18 $\pm$ 0.005 &	\cellcolor{blue!25} 0.047 $\pm$ 0.008 &	\cellcolor{red!25} 0.044 $\pm$ 0.004 \\
 &	0.7 &	0.26 $\pm$ 0.012 &	0.232 $\pm$ 0.027 &	0.235 $\pm$ 0.009 &	0.247 $\pm$ 0.01 &	\cellcolor{red!25} 0.036 $\pm$ 0.007 &	\cellcolor{blue!25} 0.048 $\pm$ 0.011 \\
\bottomrule
\end{tabular}
\begin{tabular}{l|l|llllll}
\multicolumn{8}{c}{Disease: KIDNEY} \\
\midrule
a & k & NAIVE & ExTRA \cite{MaityYurochkinBanerjeeSun2023} & BBSC \cite{LiptonWangSmola2018} & MLLS  \cite{SaerensLatinneDecaestecker2002} & SEES \cite{ChenZahariaZou2024} & CPSM \\
\midrule
0.05 &	0.3 &	0.187 $\pm$ 0.003 &	0.158 $\pm$ 0.014 &	0.108 $\pm$ 0.012 &	\cellcolor{blue!25} 0.095 $\pm$ 0.009 &	0.115 $\pm$ 0.007 &	\cellcolor{red!25} 0.059 $\pm$ 0.015 \\
 &	0.5 &	0.251 $\pm$ 0.005 &	0.215 $\pm$ 0.015 &	0.147 $\pm$ 0.012 &	0.148 $\pm$ 0.008 &	\cellcolor{blue!25} 0.115 $\pm$ 0.014 &	\cellcolor{red!25} 0.065 $\pm$ 0.014 \\
 &	0.7 &	0.292 $\pm$ 0.006 &	0.243 $\pm$ 0.015 &	0.186 $\pm$ 0.013 &	0.195 $\pm$ 0.018 &	\cellcolor{blue!25} 0.13 $\pm$ 0.02 &	\cellcolor{red!25} 0.054 $\pm$ 0.005 \\
\midrule
0.1 &	0.3 &	0.139 $\pm$ 0.003 &	0.122 $\pm$ 0.021 &	0.091 $\pm$ 0.009 &	0.08 $\pm$ 0.006 &	\cellcolor{blue!25} 0.058 $\pm$ 0.004 &	\cellcolor{red!25} 0.047 $\pm$ 0.012 \\
 &	0.5 &	0.198 $\pm$ 0.004 &	0.174 $\pm$ 0.021 &	0.139 $\pm$ 0.009 &	0.142 $\pm$ 0.005 &	\cellcolor{blue!25} 0.049 $\pm$ 0.016 &	\cellcolor{red!25} 0.046 $\pm$ 0.005 \\
 &	0.7 &	0.238 $\pm$ 0.007 &	0.205 $\pm$ 0.024 &	0.18 $\pm$ 0.006 &	0.186 $\pm$ 0.009 &	\cellcolor{blue!25} 0.047 $\pm$ 0.014 &	\cellcolor{red!25} 0.039 $\pm$ 0.007 \\
\midrule
0.2 &	0.3 &	0.095 $\pm$ 0.002 &	0.08 $\pm$ 0.021 &	0.077 $\pm$ 0.002 &	0.079 $\pm$ 0.003 &	\cellcolor{red!25} 0.027 $\pm$ 0.006 &	\cellcolor{blue!25} 0.028 $\pm$ 0.002 \\
 &	0.5 &	0.153 $\pm$ 0.003 &	0.129 $\pm$ 0.031 &	0.134 $\pm$ 0.004 &	0.132 $\pm$ 0.001 &	\cellcolor{blue!25} 0.047 $\pm$ 0.009 &	\cellcolor{red!25} 0.033 $\pm$ 0.003 \\
 &	0.7 &	0.2 $\pm$ 0.004 &	0.158 $\pm$ 0.011 &	0.176 $\pm$ 0.004 &	0.183 $\pm$ 0.004 &	\cellcolor{blue!25} 0.038 $\pm$ 0.009 &	\cellcolor{red!25} 0.032 $\pm$ 0.002 \\
\bottomrule
\end{tabular}
\begin{tabular}{l|l|llllll}
\multicolumn{8}{c}{Disease: FLUID} \\
\midrule
a & k & NAIVE & ExTRA \cite{MaityYurochkinBanerjeeSun2023} & BBSC \cite{LiptonWangSmola2018} & MLLS  \cite{SaerensLatinneDecaestecker2002} & SEES \cite{ChenZahariaZou2024} & CPSM \\
\midrule
0.05 &	0.3 &	0.235 $\pm$ 0.005 &	0.203 $\pm$ 0.023 &	0.284 $\pm$ 0.21 &	\cellcolor{blue!25} 0.151 $\pm$ 0.012 &	0.224 $\pm$ 0.011 &	\cellcolor{red!25} 0.117 $\pm$ 0.015 \\
 &	0.5 &	0.313 $\pm$ 0.007 &	0.261 $\pm$ 0.026 &	\cellcolor{blue!25} 0.224 $\pm$ 0.015 &	0.226 $\pm$ 0.01 &	0.272 $\pm$ 0.023 &	\cellcolor{red!25} 0.138 $\pm$ 0.043 \\
 &	0.7 &	0.372 $\pm$ 0.012 &	0.315 $\pm$ 0.007 &	\cellcolor{blue!25} 0.29 $\pm$ 0.009 &	0.291 $\pm$ 0.012 &	0.325 $\pm$ 0.031 &	\cellcolor{red!25} 0.144 $\pm$ 0.024 \\
\midrule
0.1 &	0.3 &	0.188 $\pm$ 0.005 &	0.154 $\pm$ 0.01 &	0.127 $\pm$ 0.01 &	0.128 $\pm$ 0.005 &	\cellcolor{blue!25} 0.123 $\pm$ 0.013 &	\cellcolor{red!25} 0.076 $\pm$ 0.01 \\
 &	0.5 &	0.267 $\pm$ 0.007 &	0.262 $\pm$ 0.017 &	0.221 $\pm$ 0.01 &	0.212 $\pm$ 0.007 &	\cellcolor{blue!25} 0.149 $\pm$ 0.017 &	\cellcolor{red!25} 0.083 $\pm$ 0.025 \\
&	0.7 &	0.326 $\pm$ 0.012 &	0.305 $\pm$ 0.021 &	0.288 $\pm$ 0.01 &	0.283 $\pm$ 0.008 &	\cellcolor{blue!25} 0.194 $\pm$ 0.028 &	\cellcolor{red!25} 0.087 $\pm$ 0.018 \\
\midrule
0.2 &	0.3 &	0.143 $\pm$ 0.003 &	0.147 $\pm$ 0.022 &	0.116 $\pm$ 0.004 &	0.118 $\pm$ 0.004 &	\cellcolor{red!25} 0.052 $\pm$ 0.006 &	\cellcolor{blue!25} 0.054 $\pm$ 0.007 \\
 &	0.5 &	0.227 $\pm$ 0.006 &	0.202 $\pm$ 0.012 &	0.204 $\pm$ 0.006 &	0.204 $\pm$ 0.005 &	\cellcolor{red!25} 0.043 $\pm$ 0.007 &	\cellcolor{blue!25} 0.054 $\pm$ 0.003 \\
 &	0.7 &	0.292 $\pm$ 0.01 &	0.271 $\pm$ 0.015 &	0.278 $\pm$ 0.008 &	0.281 $\pm$ 0.006 &	\cellcolor{red!25} 0.047 $\pm$ 0.005 &	\cellcolor{blue!25} 0.065 $\pm$ 0.004 \\
\bottomrule
\end{tabular}
\end{table}


\clearpage

\bibliography{References}

\begin{appendices}

\section{Supplementary results for the case study on the MIMIC dataset}
\label{secA1}

The Appendix contains supplementary results for the case study on the MIMIC data. Tables \ref{tab:acc_dnn} and \ref{tab:error_dnn} show the Balanced Accuracy and Approximation Error when the base classifier is a neural network and the conditioning variable is age. Tables \ref{tab:acc_logit_sex} and \ref{tab:error_logit_sex} show the Balanced Accuracy and Approximation Error when the base classifier is a logistic model and the conditioning variable is gender.


\begin{table}[ht!]
\caption{{\bf Balanced Accuracy score} for MIMIC case study, for the {\bf conditioning variable: age}. Four diseases are considered: COPD, DIABETES, KIDNEY and FLUID. The method with the highest  Balanced Accuracy is in red and the second best in blue. {\bf Neural Network} is used as a base classifier.}
\label{tab:acc_dnn}
\begin{tabular}{l|l|llllll}
\toprule 
\multicolumn{8}{c}{Disease: COPD} \\
\midrule
a & k & NAIVE & ExTRA \cite{MaityYurochkinBanerjeeSun2023} & BBSC \cite{LiptonWangSmola2018} & MLLS  \cite{SaerensLatinneDecaestecker2002} & SEES \cite{ChenZahariaZou2024} & CPSM \\
\midrule 
0.05 &	0.3 &	0.573 $\pm$ 0.016 &	\cellcolor{red!25} 0.606 $\pm$ 0.011 &	0.599 $\pm$ 0.018 &	0.584 $\pm$ 0.032 &	0.589 $\pm$ 0.016 &	\cellcolor{blue!25} 0.606 $\pm$ 0.018 \\
 &	0.5 &	0.573 $\pm$ 0.01 &	0.609 $\pm$ 0.014 &	0.603 $\pm$ 0.01 &	0.596 $\pm$ 0.015 &	\cellcolor{red!25} 0.617 $\pm$ 0.011 &	\cellcolor{blue!25} 0.612 $\pm$ 0.013 \\
 &	0.7 &	0.579 $\pm$ 0.011 &	0.602 $\pm$ 0.014 &	0.606 $\pm$ 0.021 &	0.599 $\pm$ 0.022 &	\cellcolor{red!25} 0.637 $\pm$ 0.008 &	\cellcolor{blue!25} 0.622 $\pm$ 0.017 \\
\midrule
0.1 &	0.3 &	0.644 $\pm$ 0.014 &	\cellcolor{red!25} 0.682 $\pm$ 0.024 &	0.673 $\pm$ 0.024 &	0.659 $\pm$ 0.021 &	0.658 $\pm$ 0.019 &	\cellcolor{blue!25} 0.681 $\pm$ 0.011 \\
 &	0.5 &	0.642 $\pm$ 0.025 &	0.684 $\pm$ 0.033 &	0.685 $\pm$ 0.017 &	0.674 $\pm$ 0.015 &	\cellcolor{red!25} 0.749 $\pm$ 0.009 &	\cellcolor{blue!25} 0.702 $\pm$ 0.004 \\
 &	0.7 &	0.637 $\pm$ 0.02 &	0.685 $\pm$ 0.018 &	0.694 $\pm$ 0.022 &	0.677 $\pm$ 0.02 &	\cellcolor{red!25} 0.783 $\pm$ 0.015 &	\cellcolor{blue!25} 0.752 $\pm$ 0.015 \\
\midrule
0.2 &	0.3 &	0.721 $\pm$ 0.019 &	0.746 $\pm$ 0.034 &	0.76 $\pm$ 0.022 &	0.763 $\pm$ 0.022 &	\cellcolor{blue!25} 0.793 $\pm$ 0.008 &	\cellcolor{red!25} 0.799 $\pm$ 0.018 \\
 &	0.5 &	0.737 $\pm$ 0.023 &	0.793 $\pm$ 0.015 &	0.77 $\pm$ 0.017 &	0.769 $\pm$ 0.018 &	\cellcolor{red!25} 0.883 $\pm$ 0.011 &	\cellcolor{blue!25} 0.833 $\pm$ 0.009 \\
 &	0.7 &	0.737 $\pm$ 0.019 &	0.765 $\pm$ 0.029 &	0.789 $\pm$ 0.012 &	0.779 $\pm$ 0.034 &	\cellcolor{red!25} 0.912 $\pm$ 0.01 &	\cellcolor{blue!25} 0.882 $\pm$ 0.007 \\
\bottomrule
\end{tabular}
\begin{tabular}{l|l|llllll}
\multicolumn{8}{c}{Disease: DIABETES} \\
\midrule
a & k & NAIVE & ExTRA \cite{MaityYurochkinBanerjeeSun2023} & BBSC \cite{LiptonWangSmola2018} & MLLS  \cite{SaerensLatinneDecaestecker2002} & SEES \cite{ChenZahariaZou2024} & CPSM \\
\midrule
0.05 &	0.3 &	0.616 $\pm$ 0.012 &	0.63 $\pm$ 0.008 &	0.623 $\pm$ 0.01 &	\cellcolor{red!25} 0.634 $\pm$ 0.01 &	0.622 $\pm$ 0.007 &	\cellcolor{blue!25} 0.634 $\pm$ 0.014 \\
 &	0.5 &	0.616 $\pm$ 0.008 &	0.627 $\pm$ 0.015 &	0.629 $\pm$ 0.017 &	0.619 $\pm$ 0.016 &	\cellcolor{blue!25} 0.644 $\pm$ 0.011 &	\cellcolor{red!25} 0.648 $\pm$ 0.009 \\
 &	0.7 &	0.624 $\pm$ 0.014 &	0.631 $\pm$ 0.01 &	0.64 $\pm$ 0.009 &	0.644 $\pm$ 0.018 &	\cellcolor{red!25} 0.657 $\pm$ 0.02 &	\cellcolor{blue!25} 0.657 $\pm$ 0.013 \\
\midrule
0.1 &	0.3 &	0.691 $\pm$ 0.018 &	\cellcolor{red!25} 0.719 $\pm$ 0.011 &	0.696 $\pm$ 0.023 &	0.715 $\pm$ 0.012 &	0.717 $\pm$ 0.016 &	\cellcolor{blue!25} 0.718 $\pm$ 0.022 \\
 &	0.5 &	0.689 $\pm$ 0.009 &	0.71 $\pm$ 0.02 &	0.72 $\pm$ 0.02 &	0.712 $\pm$ 0.015 &	\cellcolor{red!25} 0.757 $\pm$ 0.019 &	\cellcolor{blue!25} 0.748 $\pm$ 0.016 \\
 &	0.7 &	0.698 $\pm$ 0.017 &	0.707 $\pm$ 0.014 &	0.732 $\pm$ 0.014 &	0.733 $\pm$ 0.022 &	\cellcolor{red!25} 0.794 $\pm$ 0.022 &	\cellcolor{blue!25} 0.785 $\pm$ 0.014 \\
\midrule
0.2 &	0.3 &	0.802 $\pm$ 0.017 &	\cellcolor{blue!25} 0.829 $\pm$ 0.022 &	0.806 $\pm$ 0.016 &	0.803 $\pm$ 0.017 &	\cellcolor{red!25} 0.832 $\pm$ 0.016 &	0.824 $\pm$ 0.019 \\
&	0.5 &	0.801 $\pm$ 0.019 &	0.808 $\pm$ 0.021 &	0.828 $\pm$ 0.011 &	0.818 $\pm$ 0.024 &	\cellcolor{red!25} 0.879 $\pm$ 0.012 &	\cellcolor{blue!25} 0.843 $\pm$ 0.019 \\
 &	0.7 &	0.787 $\pm$ 0.016 &	0.821 $\pm$ 0.028 &	0.84 $\pm$ 0.024 &	0.83 $\pm$ 0.013 &	\cellcolor{red!25} 0.913 $\pm$ 0.016 &	\cellcolor{blue!25} 0.885 $\pm$ 0.005 \\
\bottomrule
\end{tabular}
\begin{tabular}{l|l|llllll}
\multicolumn{8}{c}{Disease: KIDNEY} \\
\midrule
a & k & NAIVE & ExTRA \cite{MaityYurochkinBanerjeeSun2023} & BBSC \cite{LiptonWangSmola2018} & MLLS  \cite{SaerensLatinneDecaestecker2002} & SEES \cite{ChenZahariaZou2024} & CPSM \\
\midrule
0.05 &	0.3 &	0.653 $\pm$ 0.007 &	0.677 $\pm$ 0.011 &	0.692 $\pm$ 0.016 &	\cellcolor{blue!25} 0.701 $\pm$ 0.006 &	0.677 $\pm$ 0.013 &	\cellcolor{red!25} 0.71 $\pm$ 0.014 \\
 &	0.5 &	0.65 $\pm$ 0.01 &	0.675 $\pm$ 0.01 &	0.709 $\pm$ 0.016 &	\cellcolor{blue!25} 0.712 $\pm$ 0.014 &	0.702 $\pm$ 0.016 &	\cellcolor{red!25} 0.733 $\pm$ 0.014 \\
 &	0.7 &	0.651 $\pm$ 0.015 &	0.674 $\pm$ 0.012 &	0.708 $\pm$ 0.012 &	\cellcolor{blue!25} 0.719 $\pm$ 0.015 &	0.699 $\pm$ 0.013 &	\cellcolor{red!25} 0.764 $\pm$ 0.023 \\
\midrule
0.1 &	0.3 &	0.747 $\pm$ 0.01 &	0.765 $\pm$ 0.008 &	\cellcolor{blue!25} 0.788 $\pm$ 0.016 &	0.78 $\pm$ 0.011 &	0.78 $\pm$ 0.012 &	\cellcolor{red!25} 0.797 $\pm$ 0.012 \\
 &	0.5 &	0.749 $\pm$ 0.013 &	0.765 $\pm$ 0.018 &	0.81 $\pm$ 0.004 &	0.798 $\pm$ 0.011 &	\cellcolor{blue!25} 0.83 $\pm$ 0.009 &	\cellcolor{red!25} 0.836 $\pm$ 0.015 \\
 &	0.7 &	0.745 $\pm$ 0.013 &	0.761 $\pm$ 0.011 &	0.811 $\pm$ 0.019 &	0.803 $\pm$ 0.011 &	\cellcolor{blue!25} 0.851 $\pm$ 0.014 &	\cellcolor{red!25} 0.883 $\pm$ 0.009 \\
\midrule
0.2 &	0.3 &	0.838 $\pm$ 0.017 &	0.858 $\pm$ 0.008 &	0.857 $\pm$ 0.009 &	0.858 $\pm$ 0.019 &	\cellcolor{red!25} 0.868 $\pm$ 0.014 &	\cellcolor{blue!25} 0.867 $\pm$ 0.007 \\
 &	0.5 &	0.841 $\pm$ 0.012 &	0.842 $\pm$ 0.037 &	0.867 $\pm$ 0.004 &	0.863 $\pm$ 0.008 &	\cellcolor{red!25} 0.911 $\pm$ 0.007 &	\cellcolor{blue!25} 0.895 $\pm$ 0.008 \\
 &	0.7 &	0.843 $\pm$ 0.018 &	0.861 $\pm$ 0.016 &	0.875 $\pm$ 0.005 &	0.869 $\pm$ 0.009 &	\cellcolor{red!25} 0.94 $\pm$ 0.007 &	\cellcolor{blue!25} 0.922 $\pm$ 0.004 \\
\bottomrule
\end{tabular}
\begin{tabular}{l|l|llllll}
\multicolumn{8}{c}{Disease: FLUID} \\
\midrule
a & k & NAIVE & ExTRA \cite{MaityYurochkinBanerjeeSun2023} & BBSC \cite{LiptonWangSmola2018} & MLLS  \cite{SaerensLatinneDecaestecker2002} & SEES \cite{ChenZahariaZou2024} & CPSM \\
\midrule
0.05 &	0.3 &	0.564 $\pm$ 0.016 &	\cellcolor{red!25} 0.594 $\pm$ 0.003 &	0.573 $\pm$ 0.012 &	0.562 $\pm$ 0.02 &	0.574 $\pm$ 0.009 &	\cellcolor{blue!25} 0.58 $\pm$ 0.015 \\
 &	0.5 &	0.551 $\pm$ 0.011 &	\cellcolor{red!25} 0.591 $\pm$ 0.011 &	0.576 $\pm$ 0.017 &	0.572 $\pm$ 0.023 &	0.577 $\pm$ 0.013 &	\cellcolor{blue!25} 0.58 $\pm$ 0.007 \\
 &	0.7 &	0.561 $\pm$ 0.011 &	0.586 $\pm$ 0.016 &	0.578 $\pm$ 0.015 &	0.572 $\pm$ 0.009 &	\cellcolor{blue!25} 0.591 $\pm$ 0.009 &	\cellcolor{red!25} 0.594 $\pm$ 0.01 \\
\midrule
0.1 &	0.3 &	0.623 $\pm$ 0.024 &	\cellcolor{red!25} 0.673 $\pm$ 0.008 &	0.648 $\pm$ 0.02 &	0.641 $\pm$ 0.027 &	0.65 $\pm$ 0.004 &	\cellcolor{blue!25} 0.655 $\pm$ 0.023 \\
&	0.5 &	0.619 $\pm$ 0.018 &	0.645 $\pm$ 0.026 &	0.656 $\pm$ 0.013 &	0.664 $\pm$ 0.028 &	\cellcolor{red!25} 0.682 $\pm$ 0.016 &	\cellcolor{blue!25} 0.68 $\pm$ 0.02 \\
 &	0.7 &	0.618 $\pm$ 0.027 &	0.654 $\pm$ 0.004 &	0.674 $\pm$ 0.008 &	0.664 $\pm$ 0.025 &	\cellcolor{red!25} 0.706 $\pm$ 0.016 &	\cellcolor{blue!25} 0.699 $\pm$ 0.01 \\
\midrule
0.2 &	0.3 &	0.75 $\pm$ 0.017 &	0.764 $\pm$ 0.021 &	0.751 $\pm$ 0.031 &	0.753 $\pm$ 0.018 &	\cellcolor{red!25} 0.795 $\pm$ 0.024 &	\cellcolor{blue!25} 0.789 $\pm$ 0.016 \\
 &	0.5 &	0.753 $\pm$ 0.032 &	0.764 $\pm$ 0.033 &	0.79 $\pm$ 0.014 &	0.775 $\pm$ 0.007 &	\cellcolor{red!25} 0.848 $\pm$ 0.008 &	\cellcolor{blue!25} 0.817 $\pm$ 0.016 \\
 &	0.7 &	0.741 $\pm$ 0.025 &	0.777 $\pm$ 0.012 &	0.793 $\pm$ 0.032 &	0.777 $\pm$ 0.02 &	\cellcolor{red!25} 0.894 $\pm$ 0.017 &	\cellcolor{blue!25} 0.866 $\pm$ 0.014 \\
\bottomrule
\end{tabular}
\end{table}


\begin{table}[ht!]
\caption{{\bf Approximation  Error} for MIMIC case study, for the {\bf conditioning variable: age}. Four diseases are considered: COPD, DIABETES, KIDNEY and FLUID. The method with the  smallest Approximation Error is in red and the second best in blue. {\bf Neural Network} is used as a base classifier.}
\label{tab:error_dnn}
\begin{tabular}{l|l|llllll}
\toprule 
\multicolumn{8}{c}{Disease: COPD} \\
\midrule
a & k & NAIVE & ExTRA \cite{MaityYurochkinBanerjeeSun2023} & BBSC \cite{LiptonWangSmola2018} & MLLS  \cite{SaerensLatinneDecaestecker2002} & SEES \cite{ChenZahariaZou2024} & CPSM \\
\midrule 
0.05 &	0.3 &	0.236 $\pm$ 0.004 &	0.249 $\pm$ 0.008 &	\cellcolor{blue!25} 0.234 $\pm$ 0.005 &	0.234 $\pm$ 0.006 &	0.235 $\pm$ 0.004 &	\cellcolor{red!25} 0.225 $\pm$ 0.006 \\
 &	0.5 &	0.283 $\pm$ 0.003 &	0.282 $\pm$ 0.009 &	0.274 $\pm$ 0.007 &	\cellcolor{blue!25} 0.268 $\pm$ 0.003 &	0.268 $\pm$ 0.008 &	\cellcolor{red!25} 0.258 $\pm$ 0.007 \\
 &	0.7 &	0.309 $\pm$ 0.002 &	0.312 $\pm$ 0.007 &	0.305 $\pm$ 0.012 &	0.295 $\pm$ 0.013 &	\cellcolor{blue!25} 0.273 $\pm$ 0.008 &	\cellcolor{red!25} 0.271 $\pm$ 0.012 \\
\midrule
0.1 &	0.3 &	0.196 $\pm$ 0.007 &	0.216 $\pm$ 0.012 &	0.186 $\pm$ 0.007 &	0.186 $\pm$ 0.007 &	\cellcolor{blue!25} 0.171 $\pm$ 0.003 &	\cellcolor{red!25} 0.158 $\pm$ 0.008 \\
 &	0.5 &	0.234 $\pm$ 0.01 &	0.243 $\pm$ 0.019 &	0.222 $\pm$ 0.009 &	0.215 $\pm$ 0.009 &	\cellcolor{red!25} 0.171 $\pm$ 0.011 &	\cellcolor{blue!25} 0.18 $\pm$ 0.007 \\
 &	0.7 &	0.265 $\pm$ 0.01 &	0.253 $\pm$ 0.01 &	0.245 $\pm$ 0.011 &	0.247 $\pm$ 0.007 &	\cellcolor{red!25} 0.165 $\pm$ 0.013 &	\cellcolor{blue!25} 0.18 $\pm$ 0.018 \\
\midrule
0.2 &	0.3 &	0.151 $\pm$ 0.011 &	0.194 $\pm$ 0.028 &	0.154 $\pm$ 0.005 &	0.154 $\pm$ 0.005 &	\cellcolor{red!25} 0.102 $\pm$ 0.007 &	\cellcolor{blue!25} 0.128 $\pm$ 0.005 \\
 &	0.5 &	0.188 $\pm$ 0.009 &	0.186 $\pm$ 0.015 &	0.188 $\pm$ 0.008 &	0.205 $\pm$ 0.011 &	\cellcolor{red!25} 0.097 $\pm$ 0.004 &	\cellcolor{blue!25} 0.117 $\pm$ 0.011 \\
 &	0.7 &	0.215 $\pm$ 0.008 &	0.212 $\pm$ 0.016 &	0.213 $\pm$ 0.009 &	0.239 $\pm$ 0.029 &	\cellcolor{red!25} 0.077 $\pm$ 0.008 &	\cellcolor{blue!25} 0.102 $\pm$ 0.004 \\
\bottomrule
\end{tabular}
\begin{tabular}{l|l|llllll}
\multicolumn{8}{c}{Disease: DIABETES} \\
\midrule
a & k & NAIVE & ExTRA \cite{MaityYurochkinBanerjeeSun2023} & BBSC \cite{LiptonWangSmola2018} & MLLS  \cite{SaerensLatinneDecaestecker2002} & SEES \cite{ChenZahariaZou2024} & CPSM \\
\midrule
0.05 &	0.3 &	0.26 $\pm$ 0.007 &	0.264 $\pm$ 0.006 &	0.252 $\pm$ 0.008 &	\cellcolor{blue!25} 0.245 $\pm$ 0.014 &	0.252 $\pm$ 0.008 &	\cellcolor{red!25} 0.239 $\pm$ 0.012 \\
 &	0.5 &	0.306 $\pm$ 0.01 &	0.312 $\pm$ 0.017 &	0.297 $\pm$ 0.016 &	0.298 $\pm$ 0.01 &	\cellcolor{blue!25} 0.29 $\pm$ 0.015 &	\cellcolor{red!25} 0.285 $\pm$ 0.01 \\
 &	0.7 &	0.338 $\pm$ 0.019 &	0.336 $\pm$ 0.016 &	0.326 $\pm$ 0.013 &	0.317 $\pm$ 0.012 &	\cellcolor{blue!25} 0.306 $\pm$ 0.024 &	\cellcolor{red!25} 0.301 $\pm$ 0.009 \\
\midrule
0.1 &	0.3 &	0.202 $\pm$ 0.012 &	0.218 $\pm$ 0.01 &	0.197 $\pm$ 0.012 &	\cellcolor{blue!25} 0.187 $\pm$ 0.009 &	0.188 $\pm$ 0.015 &	\cellcolor{red!25} 0.171 $\pm$ 0.012 \\
 &	0.5 &	0.243 $\pm$ 0.01 &	0.252 $\pm$ 0.015 &	0.231 $\pm$ 0.009 &	0.223 $\pm$ 0.009 &	\cellcolor{blue!25} 0.196 $\pm$ 0.018 &	\cellcolor{red!25} 0.192 $\pm$ 0.014 \\
 &	0.7 &	0.267 $\pm$ 0.011 &	0.271 $\pm$ 0.012 &	0.25 $\pm$ 0.01 &	0.242 $\pm$ 0.013 &	\cellcolor{blue!25} 0.189 $\pm$ 0.018 &	\cellcolor{red!25} 0.182 $\pm$ 0.014 \\
\midrule
0.2 &	0.3 &	0.133 $\pm$ 0.007 &	0.157 $\pm$ 0.007 &	0.137 $\pm$ 0.004 &	0.132 $\pm$ 0.01 &	\cellcolor{blue!25} 0.115 $\pm$ 0.008 &	\cellcolor{red!25} 0.109 $\pm$ 0.008 \\
 &	0.5 &	0.168 $\pm$ 0.013 &	0.185 $\pm$ 0.011 &	0.156 $\pm$ 0.011 &	0.155 $\pm$ 0.014 &	\cellcolor{red!25} 0.1 $\pm$ 0.008 &	\cellcolor{blue!25} 0.107 $\pm$ 0.002 \\
 &	0.7 &	0.199 $\pm$ 0.01 &	0.186 $\pm$ 0.029 &	0.175 $\pm$ 0.018 &	0.18 $\pm$ 0.01 &	\cellcolor{red!25} 0.081 $\pm$ 0.016 &	\cellcolor{blue!25} 0.1 $\pm$ 0.009 \\
\bottomrule
\end{tabular}
\begin{tabular}{l|l|llllll}
\multicolumn{8}{c}{Disease: KIDNEY} \\
\midrule
a & k & NAIVE & ExTRA \cite{MaityYurochkinBanerjeeSun2023} & BBSC \cite{LiptonWangSmola2018} & MLLS  \cite{SaerensLatinneDecaestecker2002} & SEES \cite{ChenZahariaZou2024} & CPSM \\
\midrule
0.05 &	0.3 &	0.217 $\pm$ 0.007 &	0.214 $\pm$ 0.006 &	0.194 $\pm$ 0.011 &	\cellcolor{blue!25} 0.188 $\pm$ 0.008 &	0.205 $\pm$ 0.01 &	\cellcolor{red!25} 0.17 $\pm$ 0.018 \\
 &	0.5 &	0.267 $\pm$ 0.011 &	0.257 $\pm$ 0.009 &	0.232 $\pm$ 0.005 &	\cellcolor{blue!25} 0.223 $\pm$ 0.017 &	0.23 $\pm$ 0.009 &	\cellcolor{red!25} 0.201 $\pm$ 0.011 \\
 &	0.7 &	0.297 $\pm$ 0.018 &	0.286 $\pm$ 0.011 &	0.262 $\pm$ 0.007 &	\cellcolor{blue!25} 0.237 $\pm$ 0.012 &	0.261 $\pm$ 0.012 &	\cellcolor{red!25} 0.198 $\pm$ 0.021 \\
\midrule
0.1 &	0.3 &	0.159 $\pm$ 0.006 &	0.163 $\pm$ 0.005 &	0.138 $\pm$ 0.01 &	\cellcolor{blue!25} 0.128 $\pm$ 0.007 &	0.132 $\pm$ 0.004 &	\cellcolor{red!25} 0.116 $\pm$ 0.012 \\
 &	0.5 &	0.198 $\pm$ 0.007 &	0.194 $\pm$ 0.012 &	0.157 $\pm$ 0.003 &	0.152 $\pm$ 0.005 &	\cellcolor{blue!25} 0.132 $\pm$ 0.009 &	\cellcolor{red!25} 0.115 $\pm$ 0.013 \\
 &	0.7 &	0.221 $\pm$ 0.012 &	0.217 $\pm$ 0.009 &	0.181 $\pm$ 0.011 &	0.176 $\pm$ 0.008 &	\cellcolor{blue!25} 0.131 $\pm$ 0.014 &	\cellcolor{red!25} 0.098 $\pm$ 0.007 \\
\midrule
0.2 &	0.3 &	0.102 $\pm$ 0.002 &	0.104 $\pm$ 0.008 &	0.092 $\pm$ 0.005 &	0.088 $\pm$ 0.003 &	\cellcolor{blue!25} 0.073 $\pm$ 0.007 &	\cellcolor{red!25} 0.068 $\pm$ 0.002 \\
 &	0.5 &	0.131 $\pm$ 0.002 &	0.135 $\pm$ 0.025 &	0.119 $\pm$ 0.006 &	0.112 $\pm$ 0.004 &	\cellcolor{red!25} 0.062 $\pm$ 0.009 &	\cellcolor{blue!25} 0.067 $\pm$ 0.005 \\
 &	0.7 &	0.155 $\pm$ 0.01 &	0.142 $\pm$ 0.01 &	0.139 $\pm$ 0.008 &	0.137 $\pm$ 0.007 &	\cellcolor{red!25} 0.045 $\pm$ 0.005 &	\cellcolor{blue!25} 0.057 $\pm$ 0.003 \\
\bottomrule
\end{tabular}
\begin{tabular}{l|l|llllll}
\multicolumn{8}{c}{Disease: FLUID} \\
\midrule
a & k & NAIVE & ExTRA \cite{MaityYurochkinBanerjeeSun2023} & BBSC \cite{LiptonWangSmola2018} & MLLS  \cite{SaerensLatinneDecaestecker2002} & SEES \cite{ChenZahariaZou2024} & CPSM \\
\midrule
0.05 &	0.3 &	0.29 $\pm$ 0.005 &	0.294 $\pm$ 0.004 &	\cellcolor{blue!25} 0.289 $\pm$ 0.009 &	0.293 $\pm$ 0.01 &	0.292 $\pm$ 0.003 &	\cellcolor{red!25} 0.281 $\pm$ 0.008 \\
 &	0.5 &	0.346 $\pm$ 0.007 &	0.343 $\pm$ 0.01 &	0.338 $\pm$ 0.014 &	\cellcolor{blue!25} 0.334 $\pm$ 0.016 &	0.343 $\pm$ 0.013 &	\cellcolor{red!25} 0.321 $\pm$ 0.021 \\
 &	0.7 &	0.387 $\pm$ 0.015 &	0.38 $\pm$ 0.013 &	0.379 $\pm$ 0.013 &	0.377 $\pm$ 0.008 &	\cellcolor{blue!25} 0.371 $\pm$ 0.016 &	\cellcolor{red!25} 0.353 $\pm$ 0.011 \\
\midrule
0.1 &	0.3 &	0.242 $\pm$ 0.011 &	0.25 $\pm$ 0.008 &	\cellcolor{blue!25} 0.228 $\pm$ 0.008 &	0.236 $\pm$ 0.01 &	0.232 $\pm$ 0.007 &	\cellcolor{red!25} 0.217 $\pm$ 0.016 \\
 &	0.5 &	0.293 $\pm$ 0.009 &	0.301 $\pm$ 0.019 &	0.274 $\pm$ 0.007 &	0.26 $\pm$ 0.015 &	\cellcolor{blue!25} 0.256 $\pm$ 0.013 &	\cellcolor{red!25} 0.235 $\pm$ 0.016 \\
 &	0.7 &	0.333 $\pm$ 0.02 &	0.32 $\pm$ 0.007 &	0.303 $\pm$ 0.011 &	0.298 $\pm$ 0.01 &	\cellcolor{blue!25} 0.269 $\pm$ 0.021 &	\cellcolor{red!25} 0.25 $\pm$ 0.009 \\
\midrule
0.2 &	0.3 &	0.169 $\pm$ 0.009 &	0.196 $\pm$ 0.02 &	0.16 $\pm$ 0.008 &	0.16 $\pm$ 0.012 &	\cellcolor{blue!25} 0.141 $\pm$ 0.008 &	\cellcolor{red!25} 0.129 $\pm$ 0.006 \\
 &	0.5 &	0.208 $\pm$ 0.016 &	0.21 $\pm$ 0.011 &	0.196 $\pm$ 0.01 &	0.189 $\pm$ 0.009 &	\cellcolor{red!25} 0.123 $\pm$ 0.004 &	\cellcolor{blue!25} 0.125 $\pm$ 0.009 \\
 &	0.7 &	0.239 $\pm$ 0.015 &	0.232 $\pm$ 0.008 &	0.223 $\pm$ 0.018 &	0.227 $\pm$ 0.016 &	\cellcolor{red!25} 0.095 $\pm$ 0.012 &	\cellcolor{blue!25} 0.119 $\pm$ 0.01 \\
\bottomrule
\end{tabular}
\end{table}


\begin{table}[ht!]
\caption{{\bf Balanced Accuracy score} for MIMIC case study, for the {\bf conditioning variable: gender}. Four diseases are considered: COPD, DIABETES, KIDNEY and FLUID. The method with the highest Balanced Accuracy is in red and the second best in blue. {\bf Logistic regression} is used as a base classifier.}
\label{tab:acc_logit_sex}
\begin{tabular}{l|l|llllll}
\toprule 
\multicolumn{8}{c}{Disease: COPD} \\
\midrule
a & k & NAIVE & ExTRA \cite{MaityYurochkinBanerjeeSun2023} & BBSC \cite{LiptonWangSmola2018} & MLLS  \cite{SaerensLatinneDecaestecker2002} & SEES \cite{ChenZahariaZou2024} & CPSM \\
\midrule 
0.05 &	0.3 &	0.51 $\pm$ 0.004 &	0.555 $\pm$ 0.029 &	0.553 $\pm$ 0.01 &	\cellcolor{blue!25} 0.566 $\pm$ 0.016 &	0.52 $\pm$ 0.006 &	\cellcolor{red!25} 0.593 $\pm$ 0.005 \\
 &	0.5 &	0.51 $\pm$ 0.004 &	0.559 $\pm$ 0.03 &	\cellcolor{blue!25} 0.577 $\pm$ 0.017 &	0.572 $\pm$ 0.016 &	0.548 $\pm$ 0.023 &	\cellcolor{red!25} 0.651 $\pm$ 0.017 \\
 &	0.7 &	0.511 $\pm$ 0.005 &	0.573 $\pm$ 0.016 &	\cellcolor{blue!25} 0.579 $\pm$ 0.013 &	0.578 $\pm$ 0.013 &	0.57 $\pm$ 0.032 &	\cellcolor{red!25} 0.704 $\pm$ 0.02 \\
\midrule
0.1 &	0.3 &	0.53 $\pm$ 0.006 &	\cellcolor{blue!25} 0.589 $\pm$ 0.022 &	0.563 $\pm$ 0.008 &	0.578 $\pm$ 0.01 &	0.561 $\pm$ 0.011 &	\cellcolor{red!25} 0.616 $\pm$ 0.008 \\
 &	0.5 &	0.527 $\pm$ 0.005 &	0.567 $\pm$ 0.024 &	0.592 $\pm$ 0.009 &	0.591 $\pm$ 0.014 &	\cellcolor{blue!25} 0.644 $\pm$ 0.022 &	\cellcolor{red!25} 0.693 $\pm$ 0.017 \\
 &	0.7 &	0.532 $\pm$ 0.009 &	0.546 $\pm$ 0.029 &	0.588 $\pm$ 0.008 &	0.587 $\pm$ 0.016 &	\cellcolor{blue!25} 0.716 $\pm$ 0.029 &	\cellcolor{red!25} 0.741 $\pm$ 0.015 \\
\midrule
0.2 &	0.3 &	0.575 $\pm$ 0.008 &	0.565 $\pm$ 0.037 &	0.579 $\pm$ 0.011 &	0.579 $\pm$ 0.009 &	\cellcolor{blue!25} 0.609 $\pm$ 0.004 &	\cellcolor{red!25} 0.633 $\pm$ 0.016 \\
 &	0.5 &	0.578 $\pm$ 0.007 &	0.575 $\pm$ 0.036 &	0.589 $\pm$ 0.01 &	0.59 $\pm$ 0.018 &	\cellcolor{red!25} 0.728 $\pm$ 0.033 &	\cellcolor{blue!25} 0.711 $\pm$ 0.017 \\
 &	0.7 &	0.579 $\pm$ 0.012 &	0.582 $\pm$ 0.051 &	0.605 $\pm$ 0.014 &	0.611 $\pm$ 0.011 &	\cellcolor{red!25} 0.812 $\pm$ 0.007 &	\cellcolor{blue!25} 0.776 $\pm$ 0.012 \\
\bottomrule
\end{tabular}
\begin{tabular}{l|l|llllll}
\multicolumn{8}{c}{Disease: DIABETES} \\
\midrule
a & k & NAIVE & ExTRA \cite{MaityYurochkinBanerjeeSun2023} & BBSC \cite{LiptonWangSmola2018} & MLLS  \cite{SaerensLatinneDecaestecker2002} & SEES \cite{ChenZahariaZou2024} & CPSM \\
\midrule
0.05 &	0.3 &	0.555 $\pm$ 0.014 &	0.582 $\pm$ 0.016 &	\cellcolor{blue!25} 0.631 $\pm$ 0.012 &	0.622 $\pm$ 0.019 &	0.574 $\pm$ 0.029 &	\cellcolor{red!25} 0.652 $\pm$ 0.03 \\
 &	0.5 &	0.553 $\pm$ 0.007 &	0.571 $\pm$ 0.043 &	\cellcolor{blue!25} 0.645 $\pm$ 0.027 &	0.629 $\pm$ 0.022 &	0.632 $\pm$ 0.042 &	\cellcolor{red!25} 0.734 $\pm$ 0.036 \\
 &	0.7 &	0.549 $\pm$ 0.015 &	0.582 $\pm$ 0.032 &	0.642 $\pm$ 0.021 &	0.642 $\pm$ 0.024 &	\cellcolor{blue!25} 0.716 $\pm$ 0.044 &	\cellcolor{red!25} 0.793 $\pm$ 0.04 \\
\midrule
0.1 &	0.3 &	0.586 $\pm$ 0.017 &	0.589 $\pm$ 0.019 &	0.642 $\pm$ 0.019 &	\cellcolor{blue!25} 0.646 $\pm$ 0.018 &	0.635 $\pm$ 0.026 &	\cellcolor{red!25} 0.678 $\pm$ 0.025 \\
 &	0.5 &	0.587 $\pm$ 0.015 &	0.646 $\pm$ 0.043 &	0.661 $\pm$ 0.026 &	0.65 $\pm$ 0.016 &	\cellcolor{blue!25} 0.734 $\pm$ 0.03 &	\cellcolor{red!25} 0.76 $\pm$ 0.03 \\
 &	0.7 &	0.578 $\pm$ 0.018 &	0.612 $\pm$ 0.017 &	0.668 $\pm$ 0.018 &	0.67 $\pm$ 0.022 &	\cellcolor{red!25} 0.829 $\pm$ 0.011 &	\cellcolor{blue!25} 0.825 $\pm$ 0.015 \\
\midrule
0.2 &	0.3 &	0.644 $\pm$ 0.006 &	0.663 $\pm$ 0.019 &	0.649 $\pm$ 0.025 &	0.649 $\pm$ 0.018 &	\cellcolor{blue!25} 0.677 $\pm$ 0.009 &	\cellcolor{red!25} 0.691 $\pm$ 0.023 \\
 &	0.5 &	0.642 $\pm$ 0.015 &	0.635 $\pm$ 0.054 &	0.664 $\pm$ 0.014 &	0.665 $\pm$ 0.014 &	\cellcolor{red!25} 0.801 $\pm$ 0.018 &	\cellcolor{blue!25} 0.789 $\pm$ 0.022 \\
 &	0.7 &	0.645 $\pm$ 0.012 &	0.67 $\pm$ 0.032 &	0.674 $\pm$ 0.019 &	0.68 $\pm$ 0.013 &	\cellcolor{red!25} 0.854 $\pm$ 0.006 &	\cellcolor{blue!25} 0.845 $\pm$ 0.009 \\
\bottomrule
\end{tabular}
\begin{tabular}{l|l|llllll}
\multicolumn{8}{c}{Disease: KIDNEY} \\
\midrule
a & k & NAIVE & ExTRA \cite{MaityYurochkinBanerjeeSun2023} & BBSC \cite{LiptonWangSmola2018} & MLLS  \cite{SaerensLatinneDecaestecker2002} & SEES \cite{ChenZahariaZou2024} & CPSM \\
\midrule
0.05 &	0.3 &	0.6 $\pm$ 0.014 &	0.661 $\pm$ 0.022 &	\cellcolor{blue!25} 0.728 $\pm$ 0.015 &	0.706 $\pm$ 0.019 &	0.698 $\pm$ 0.033 &	\cellcolor{red!25} 0.78 $\pm$ 0.018 \\
 &	0.5 &	0.604 $\pm$ 0.02 &	0.677 $\pm$ 0.027 &	\cellcolor{blue!25} 0.746 $\pm$ 0.011 &	0.727 $\pm$ 0.015 &	0.742 $\pm$ 0.03 &	\cellcolor{red!25} 0.818 $\pm$ 0.012 \\
 &	0.7 &	0.606 $\pm$ 0.013 &	0.648 $\pm$ 0.021 &	0.722 $\pm$ 0.005 &	0.734 $\pm$ 0.019 &	\cellcolor{blue!25} 0.795 $\pm$ 0.025 &	\cellcolor{red!25} 0.847 $\pm$ 0.008 \\
\midrule
0.1 &	0.3 &	0.664 $\pm$ 0.013 &	0.714 $\pm$ 0.033 &	0.744 $\pm$ 0.014 &	0.735 $\pm$ 0.013 &	\cellcolor{blue!25} 0.766 $\pm$ 0.025 &	\cellcolor{red!25} 0.789 $\pm$ 0.015 \\
 &	0.5 &	0.665 $\pm$ 0.012 &	0.703 $\pm$ 0.035 &	0.761 $\pm$ 0.019 &	0.754 $\pm$ 0.015 &	\cellcolor{blue!25} 0.823 $\pm$ 0.006 &	\cellcolor{red!25} 0.836 $\pm$ 0.009 \\
 &	0.7 &	0.663 $\pm$ 0.008 &	0.673 $\pm$ 0.04 &	0.758 $\pm$ 0.009 &	0.753 $\pm$ 0.013 &	\cellcolor{blue!25} 0.861 $\pm$ 0.015 &	\cellcolor{red!25} 0.865 $\pm$ 0.008 \\
\midrule
0.2 &	0.3 &	0.741 $\pm$ 0.013 &	0.675 $\pm$ 0.205 &	0.755 $\pm$ 0.009 &	0.757 $\pm$ 0.011 &	\cellcolor{blue!25} 0.793 $\pm$ 0.013 &	\cellcolor{red!25} 0.797 $\pm$ 0.009 \\
 &	0.5 &	0.745 $\pm$ 0.009 &	0.757 $\pm$ 0.014 &	0.774 $\pm$ 0.012 &	0.768 $\pm$ 0.015 &	\cellcolor{blue!25} 0.848 $\pm$ 0.007 &	\cellcolor{red!25} 0.849 $\pm$ 0.006 \\
 &	0.7 &	0.745 $\pm$ 0.007 &	0.774 $\pm$ 0.024 &	0.777 $\pm$ 0.011 &	0.773 $\pm$ 0.007 &	\cellcolor{red!25} 0.876 $\pm$ 0.01 &	\cellcolor{blue!25} 0.875 $\pm$ 0.008 \\
\bottomrule
\end{tabular}
\begin{tabular}{l|l|llllll}
\multicolumn{8}{c}{Disease: FLUID} \\
\midrule
a & k & NAIVE & ExTRA \cite{MaityYurochkinBanerjeeSun2023} & BBSC \cite{LiptonWangSmola2018} & MLLS  \cite{SaerensLatinneDecaestecker2002} & SEES \cite{ChenZahariaZou2024} & CPSM \\
\midrule
0.05 &	0.3 &	0.509 $\pm$ 0.002 &	\cellcolor{blue!25} 0.551 $\pm$ 0.017 &	0.534 $\pm$ 0.006 &	0.55 $\pm$ 0.012 &	0.508 $\pm$ 0.005 &	\cellcolor{red!25} 0.561 $\pm$ 0.011 \\
 &	0.5 &	0.508 $\pm$ 0.003 &	0.537 $\pm$ 0.023 &	\cellcolor{blue!25} 0.553 $\pm$ 0.011 &	0.548 $\pm$ 0.012 &	0.513 $\pm$ 0.005 &	\cellcolor{red!25} 0.635 $\pm$ 0.022 \\
 &	0.7 &	0.508 $\pm$ 0.002 &	0.566 $\pm$ 0.009 &	\cellcolor{blue!25} 0.575 $\pm$ 0.017 &	0.553 $\pm$ 0.01 &	0.529 $\pm$ 0.014 &	\cellcolor{red!25} 0.743 $\pm$ 0.031 \\
\midrule
0.1 &	0.3 &	0.521 $\pm$ 0.006 &	0.544 $\pm$ 0.016 &	0.557 $\pm$ 0.007 &	\cellcolor{blue!25} 0.561 $\pm$ 0.01 &	0.546 $\pm$ 0.011 &	\cellcolor{red!25} 0.598 $\pm$ 0.016 \\
 &	0.5 &	0.52 $\pm$ 0.002 &	0.527 $\pm$ 0.06 &	0.568 $\pm$ 0.011 &	\cellcolor{blue!25} 0.571 $\pm$ 0.005 &	0.57 $\pm$ 0.019 &	\cellcolor{red!25} 0.707 $\pm$ 0.039 \\
 &	0.7 &	0.523 $\pm$ 0.006 &	0.576 $\pm$ 0.023 &	0.586 $\pm$ 0.01 &	0.576 $\pm$ 0.004 &	\cellcolor{blue!25} 
0.712 $\pm$ 0.024 &	\cellcolor{red!25} 0.792 $\pm$ 0.028 \\
\midrule
0.2 &	0.3 &	0.57 $\pm$ 0.013 &	0.522 $\pm$ 0.08 &	0.572 $\pm$ 0.012 &	0.57 $\pm$ 0.014 &	\cellcolor{blue!25} 0.607 $\pm$ 0.012 &	\cellcolor{red!25} 0.624 $\pm$ 0.012 \\
 &	0.5 &	0.567 $\pm$ 0.01 &	0.584 $\pm$ 0.06 &	0.598 $\pm$ 0.006 &	0.587 $\pm$ 0.01 &	\cellcolor{blue!25} 0.683 $\pm$ 0.033 &	\cellcolor{red!25} 0.735 $\pm$ 0.022 \\
 &	0.7 &	0.567 $\pm$ 0.008 &	0.569 $\pm$ 0.039 &	0.6 $\pm$ 0.011 &	0.603 $\pm$ 0.011 &	\cellcolor{red!25} 0.831 $\pm$ 0.011 &	\cellcolor{blue!25} 0.828 $\pm$ 0.011 \\
\bottomrule
\end{tabular}
\end{table}


\begin{table}[ht!]
\caption{{\bf Approximation  Error} for MIMIC case study, for the {\bf conditioning variable: gender}. Four diseases are considered: COPD, DIABETES, KIDNEY and FLUID. The method with the smallest approximation  error is in red and the second best in blue. {\bf Logistic regression} is used as a base classifier.}
\label{tab:error_logit_sex}
\begin{tabular}{l|l|llllll}
\toprule 
\multicolumn{8}{c}{Disease: COPD} \\
\midrule
a & k & NAIVE & ExTRA \cite{MaityYurochkinBanerjeeSun2023} & BBSC \cite{LiptonWangSmola2018} & MLLS  \cite{SaerensLatinneDecaestecker2002} & SEES \cite{ChenZahariaZou2024} & CPSM \\
\midrule 
0.05 &	0.3 &	0.149 $\pm$ 0.004 &	0.131 $\pm$ 0.01 &	0.117 $\pm$ 0.004 &	\cellcolor{blue!25} 0.112 $\pm$ 0.003 &	0.148 $\pm$ 0.005 &	\cellcolor{red!25} 0.074 $\pm$ 0.006 \\
 &	0.5 &	0.173 $\pm$ 0.005 &	0.173 $\pm$ 0.037 &	0.161 $\pm$ 0.011 &	\cellcolor{blue!25} 0.152 $\pm$ 0.009 &	0.157 $\pm$ 0.023 &	\cellcolor{red!25} 0.073 $\pm$ 0.015 \\
 &	0.7 &	0.188 $\pm$ 0.006 &	0.176 $\pm$ 0.014 &	0.194 $\pm$ 0.024 &	0.178 $\pm$ 0.013 &	\cellcolor{blue!25} 0.168 $\pm$ 0.011 &	\cellcolor{red!25} 0.087 $\pm$ 0.013 \\
\midrule
0.1 &	0.3 &	0.11 $\pm$ 0.004 &	0.111 $\pm$ 0.017 &	0.117 $\pm$ 0.006 &	0.117 $\pm$ 0.005 &	\cellcolor{blue!25} 0.091 $\pm$ 0.009 &	\cellcolor{red!25} 0.058 $\pm$ 0.003 \\
 &	0.5 &	0.137 $\pm$ 0.006 &	0.153 $\pm$ 0.013 &	0.164 $\pm$ 0.01 &	0.165 $\pm$ 0.008 &	\cellcolor{blue!25} 0.083 $\pm$ 0.015 &	\cellcolor{red!25} 0.058 $\pm$ 0.009 \\
 &	0.7 &	0.153 $\pm$ 0.005 &	0.188 $\pm$ 0.019 &	0.197 $\pm$ 0.013 &	0.182 $\pm$ 0.012 &	\cellcolor{blue!25} 0.086 $\pm$ 0.008 &	\cellcolor{red!25} 0.073 $\pm$ 0.007 \\
\midrule
0.2 &	0.3 &	0.12 $\pm$ 0.004 &	0.146 $\pm$ 0.034 &	0.119 $\pm$ 0.002 &	0.119 $\pm$ 0.002 &	\cellcolor{red!25} 0.037 $\pm$ 0.003 &	\cellcolor{blue!25} 0.044 $\pm$ 0.007 \\
 &	0.5 &	0.159 $\pm$ 0.003 &	0.191 $\pm$ 0.02 &	0.168 $\pm$ 0.008 &	0.169 $\pm$ 0.007 &	\cellcolor{red!25} 0.037 $\pm$ 0.011 &	\cellcolor{blue!25} 0.052 $\pm$ 0.009 \\
 &	0.7 &	0.18 $\pm$ 0.004 &	0.224 $\pm$ 0.022 &	0.206 $\pm$ 0.013 &	0.207 $\pm$ 0.006 &	\cellcolor{red!25} 0.028 $\pm$ 0.005 &	\cellcolor{blue!25} 0.071 $\pm$ 0.009 \\
\bottomrule
\end{tabular}
\begin{tabular}{l|l|llllll}
\multicolumn{8}{c}{Disease: DIABETES} \\
\midrule
a & k & NAIVE & ExTRA \cite{MaityYurochkinBanerjeeSun2023} & BBSC \cite{LiptonWangSmola2018} & MLLS  \cite{SaerensLatinneDecaestecker2002} & SEES \cite{ChenZahariaZou2024} & CPSM \\
\midrule
0.05 &	0.3 &	0.156 $\pm$ 0.006 &	0.212 $\pm$ 0.148 &	\cellcolor{blue!25} 0.115 $\pm$ 0.002 &	0.118 $\pm$ 0.007 &	0.134 $\pm$ 0.022 &	\cellcolor{red!25} 0.067 $\pm$ 0.013 \\
 &	0.5 &	0.189 $\pm$ 0.008 &	0.279 $\pm$ 0.206 &	0.162 $\pm$ 0.014 &	0.155 $\pm$ 0.009 &	\cellcolor{blue!25} 0.126 $\pm$ 0.018 &	\cellcolor{red!25} 0.07 $\pm$ 0.012 \\
 &	0.7 &	0.212 $\pm$ 0.01 &	0.195 $\pm$ 0.015 &	0.186 $\pm$ 0.014 &	0.194 $\pm$ 0.008 &	\cellcolor{blue!25} 0.111 $\pm$ 0.02 &	\cellcolor{red!25} 0.073 $\pm$ 0.021 \\
\midrule
0.1 &	0.3 &	0.117 $\pm$ 0.005 &	0.125 $\pm$ 0.014 &	0.109 $\pm$ 0.003 &	0.111 $\pm$ 0.005 &	\cellcolor{blue!25} 0.071 $\pm$ 0.013 &	\cellcolor{red!25} 0.05 $\pm$ 0.006 \\
 &	0.5 &	0.149 $\pm$ 0.009 &	0.149 $\pm$ 0.026 &	0.156 $\pm$ 0.013 &	0.155 $\pm$ 0.01 &	\cellcolor{blue!25} 0.061 $\pm$ 0.013 &	\cellcolor{red!25} 0.055 $\pm$ 0.004 \\
 &	0.7 &	0.174 $\pm$ 0.006 &	0.175 $\pm$ 0.017 &	0.189 $\pm$ 0.012 &	0.194 $\pm$ 0.01 &	\cellcolor{red!25} 0.047 $\pm$ 0.008 &	\cellcolor{blue!25} 0.057 $\pm$ 0.003 \\
\midrule
0.2 &	0.3 &	0.105 $\pm$ 0.005 &	0.136 $\pm$ 0.088 &	0.105 $\pm$ 0.005 &	0.107 $\pm$ 0.002 &	\cellcolor{red!25} 0.031 $\pm$ 0.004 &	\cellcolor{blue!25} 0.033 $\pm$ 0.004 \\
 &	0.5 &	0.15 $\pm$ 0.004 &	0.227 $\pm$ 0.185 &	0.158 $\pm$ 0.007 &	0.156 $\pm$ 0.006 &	\cellcolor{blue!25} 0.045 $\pm$ 0.012 &	\cellcolor{red!25} 0.042 $\pm$ 0.006 \\
 &	0.7 &	0.171 $\pm$ 0.007 &	0.165 $\pm$ 0.025 &	0.189 $\pm$ 0.007 &	0.192 $\pm$ 0.013 &	\cellcolor{red!25} 0.036 $\pm$ 0.009 &	\cellcolor{blue!25} 0.041 $\pm$ 0.006 \\
\bottomrule
\end{tabular}
\begin{tabular}{l|l|llllll}
\multicolumn{8}{c}{Disease: KIDNEY} \\
\midrule
a & k & NAIVE & ExTRA \cite{MaityYurochkinBanerjeeSun2023} & BBSC \cite{LiptonWangSmola2018} & MLLS  \cite{SaerensLatinneDecaestecker2002} & SEES \cite{ChenZahariaZou2024} & CPSM \\
\midrule
0.05 &	0.3 &	0.136 $\pm$ 0.004 &	0.112 $\pm$ 0.011 &	0.092 $\pm$ 0.007 &	0.092 $\pm$ 0.009 &	\cellcolor{blue!25} 0.08 $\pm$ 0.013 &	\cellcolor{red!25} 0.047 $\pm$ 0.008 \\
 &	0.5 &	0.16 $\pm$ 0.007 &	0.136 $\pm$ 0.013 &	0.125 $\pm$ 0.005 &	0.128 $\pm$ 0.011 &	\cellcolor{blue!25} 0.088 $\pm$ 0.012 &	\cellcolor{red!25} 0.047 $\pm$ 0.008 \\
 &	0.7 &	0.179 $\pm$ 0.004 &	0.161 $\pm$ 0.011 &	0.145 $\pm$ 0.01 &	0.148 $\pm$ 0.007 &	\cellcolor{blue!25} 0.081 $\pm$ 0.013 &	\cellcolor{red!25} 0.049 $\pm$ 0.003 \\
\midrule
0.1 &	0.3 &	0.097 $\pm$ 0.003 &	0.083 $\pm$ 0.01 &	0.086 $\pm$ 0.005 &	0.084 $\pm$ 0.008 &	\cellcolor{blue!25} 0.04 $\pm$ 0.009 &	\cellcolor{red!25} 0.039 $\pm$ 0.004 \\
 &	0.5 &	0.121 $\pm$ 0.005 &	0.119 $\pm$ 0.019 &	0.123 $\pm$ 0.002 &	0.121 $\pm$ 0.008 &	\cellcolor{blue!25} 0.039 $\pm$ 0.01 &	\cellcolor{red!25} 0.035 $\pm$ 0.006 \\
 &	0.7 &	0.142 $\pm$ 0.005 &	0.151 $\pm$ 0.022 &	0.139 $\pm$ 0.005 &	0.144 $\pm$ 0.007 &	\cellcolor{blue!25} 0.04 $\pm$ 0.001 &	\cellcolor{red!25} 0.033 $\pm$ 0.002 \\
\midrule
0.2 &	0.3 &	0.08 $\pm$ 0.001 &	0.144 $\pm$ 0.151 &	0.081 $\pm$ 0.002 &	0.082 $\pm$ 0.002 &	\cellcolor{red!25} 0.021 $\pm$ 0.002 &	\cellcolor{blue!25} 0.025 $\pm$ 0.005 \\
 &	0.5 &	0.111 $\pm$ 0.004 &	0.097 $\pm$ 0.008 &	0.119 $\pm$ 0.003 &	0.118 $\pm$ 0.003 &	\cellcolor{blue!25} 0.043 $\pm$ 0.008 &	\cellcolor{red!25} 0.03 $\pm$ 0.004 \\
 &	0.7 &	0.131 $\pm$ 0.007 &	0.125 $\pm$ 0.006 &	0.14 $\pm$ 0.006 &	0.141 $\pm$ 0.006 &	\cellcolor{blue!25} 0.033 $\pm$ 0.003 &	\cellcolor{red!25} 0.03 $\pm$ 0.004 \\
\bottomrule
\end{tabular}
\begin{tabular}{l|l|llllll}
\multicolumn{8}{c}{Disease: FLUID} \\
\midrule
a & k & NAIVE & ExTRA \cite{MaityYurochkinBanerjeeSun2023} & BBSC \cite{LiptonWangSmola2018} & MLLS  \cite{SaerensLatinneDecaestecker2002} & SEES \cite{ChenZahariaZou2024} & CPSM \\
\midrule
0.05 &	0.3 &	0.179 $\pm$ 0.004 &	0.16 $\pm$ 0.006 &	0.141 $\pm$ 0.005 &	\cellcolor{blue!25} 0.136 $\pm$ 0.005 &	0.191 $\pm$ 0.009 &	\cellcolor{red!25} 0.106 $\pm$ 0.011 \\
 &	0.5 &	0.221 $\pm$ 0.005 &	0.201 $\pm$ 0.012 &	0.188 $\pm$ 0.008 &	\cellcolor{blue!25} 0.186 $\pm$ 0.01 &	0.23 $\pm$ 0.005 &	\cellcolor{red!25} 0.106 $\pm$ 0.008 \\
 &	0.7 &	0.246 $\pm$ 0.007 &	0.225 $\pm$ 0.01 &	0.223 $\pm$ 0.015 &	\cellcolor{blue!25} 0.219 $\pm$ 0.009 &	0.24 $\pm$ 0.01 &	\cellcolor{red!25} 0.101 $\pm$ 0.014 \\
\midrule
0.1 &	0.3 &	0.141 $\pm$ 0.002 &	0.141 $\pm$ 0.011 &	0.128 $\pm$ 0.004 &	0.129 $\pm$ 0.001 &	\cellcolor{blue!25} 0.116 $\pm$ 0.007 &	\cellcolor{red!25} 0.073 $\pm$ 0.007 \\
 &	0.5 &	0.18 $\pm$ 0.004 &	0.212 $\pm$ 0.088 &	0.187 $\pm$ 0.011 &	0.186 $\pm$ 0.008 &	\cellcolor{blue!25} 0.144 $\pm$ 0.009 &	\cellcolor{red!25} 0.07 $\pm$ 0.015 \\
 &	0.7 &	0.209 $\pm$ 0.005 &	0.204 $\pm$ 0.003 &	0.222 $\pm$ 0.011 &	0.222 $\pm$ 0.011 &	\cellcolor{blue!25} 0.124 $\pm$ 0.013 &	\cellcolor{red!25} 0.078 $\pm$ 0.013 \\
\midrule
0.2 &	0.3 &	0.124 $\pm$ 0.003 &	0.182 $\pm$ 0.093 &	0.127 $\pm$ 0.003 &	0.127 $\pm$ 0.002 &	\cellcolor{red!25} 0.053 $\pm$ 0.003 &	\cellcolor{blue!25} 0.053 $\pm$ 0.005 \\
 &	0.5 &	0.179 $\pm$ 0.005 &	0.191 $\pm$ 0.061 &	0.187 $\pm$ 0.009 &	0.185 $\pm$ 0.003 &	\cellcolor{blue!25} 0.069 $\pm$ 0.013 &	\cellcolor{red!25} 0.058 $\pm$ 0.009 \\
 &	0.7 &	0.212 $\pm$ 0.005 &	0.243 $\pm$ 0.049 &	0.228 $\pm$ 0.007 &	0.232 $\pm$ 0.007 &	\cellcolor{red!25} 0.043 $\pm$ 0.009 &	\cellcolor{blue!25} 0.06 $\pm$ 0.006 \\
\bottomrule
\end{tabular}
\end{table}






\end{appendices}

\end{document}